\documentclass[11pt]{article}
\usepackage[margin=1in]{geometry}
\usepackage[T1]{fontenc}
\usepackage{mathptmx}
\usepackage{microtype}
\usepackage{amsmath,amssymb,amsthm}
\usepackage{graphicx}
\usepackage{booktabs}
\usepackage[numbers,sort&compress]{natbib}
\usepackage{xcolor}
\usepackage[colorlinks=true,linkcolor=blue!50!black,citecolor=blue!50!black,urlcolor=blue!50!black]{hyperref}
\usepackage{enumitem}
\setlist{itemsep=2pt,topsep=4pt}
\newtheorem{proposition}{Proposition}
\usepackage{caption}
\usepackage{ms_style}
\shorttitle{Rubrics Carry the Standard}
\kicker{Microsoft \textperiodcentered{} Preprint}

\title{Rubrics Carry the Standard:\\Shapley Attribution for Checklist-Based LLM Judges}
\author{Saeid Asgari, Rui Ying, Sara Malvar, Rakshanda Agarwal, Bruce Sun, Claire Wang, \\ Leonardo de Oliveira Nunes, Ranveer Chandra \\   \ \ \\
\texttt{saeidasgari@microsoft.com}}

\begin{document}
\maketitle

\begin{abstract}
Checklist rubrics applied by large language model (LLM) judges have become a standard instrument for evaluating open-ended outputs and for rewarding them during training. Two assumptions surround this practice: that leniency or strictness is a property of the judge model, and that a rubric can be improved by asking a model to review it. We test both in a controlled grading task with planted flaws and on two human-labeled datasets, SummEdits and RAGTruth. Holding the judge and the outputs fixed, rubric content moves the pass rate of correct outputs from 2\% to 100\%, whereas changing the judge model or the rendering of the same checks moves it by a few points. Model self-review then drifts the rubric toward the reviser's own standard: it adds requirements the task never stated and deletes checks that human labels had taught, with three reviser models and on all three datasets. To measure and control this, we show that under all-must-pass aggregation the Shapley value of each check has an exact closed form, which yields per-check catch credit, false-fail blame and unique catches in linear time. Attribution shows that self-review deletes the checks that do unique work and keeps the checks that cause false fails. A change-control gate that combines these quantities with a provenance freeze on label-derived checks holds nearly every drifting revision chain at its starting quality, with identical decisions under three judge models.
\end{abstract}

\section{Introduction}

A large language model (LLM) applying a rubric has become the default grader for open-ended work: agent trajectories, written reports, summaries, medical answers \citep{zheng2023judging,arora2025healthbench}. Rubrics are increasingly written as checklists of atomic yes/no questions, both for evaluation \citep{cook2024tick,lee2024checkeval} and as reward signals for reinforcement learning \citep{gunjal2025rubrics,viswanathan2025checklists}. Two assumptions have grown up around this practice. The first is that leniency or strictness belongs to the judge: when grading looks wrong, the remedy is a stronger model or a sterner prompt. The second is that a rubric can be improved by asking a model to review it for ambiguity and gaps, a step that is cheap, feels like diligence, and is often run in a loop.

This paper tests both assumptions under controlled conditions and finds that each misplaces the standard. The standard a grader enforces is written in the rubric's content; the judge largely executes it, and unsupervised model revision moves it.

Our primary testbed is a security reconnaissance task: an agent inspects a synthetic cloud tenant and writes a report of security posture issues, and a judge grades the report against the tenant's known inventory. Because every flaw is planted, every report carries a label, and because we author the rubrics, we can vary one factor at a time. We replicate the central findings on two public datasets with human labels: SummEdits \citep{laban2023llms} for summary consistency and RAGTruth \citep{wu2024ragtruth} for faithfulness in retrieval-augmented generation (RAG).

Four findings organize the paper. First, rubric content sets strictness (Section~\ref{sec:strict}). Hand-written checks pass 96 to 98\% of correct reports and model-generated checks for the same environments pass 2 to 8\%, under every rendering we tried and under three judge models that agree within a few points. The generated checks are not harsher in tone; they require content the task never asked for, and under all-must-pass aggregation each such check multiplies the pass rate of correct outputs by its own. On RAGTruth, ten generic faithfulness checks pass 41\% of responses that human annotators judged faithful, against 81\% for a free-form verdict. A validation set dominated by planted defects hides all of this, since a rubric that fails everything scores 89\% accuracy on a set that is 90\% flawed.

Second, for all-must-pass rubrics the Shapley value \citep{shapley1953value} of each check has an exact closed form (Section~\ref{sec:shapley}): a report's outcome is divided equally among the checks that failed it. This gives every check a catch credit, a false-fail blame and a count of unique catches at the cost of one pass over the judge's answers, and it decomposes a rubric's two headline rates exactly.

Third, model self-review drifts a rubric toward the reviser's own standard (Section~\ref{sec:drift}). From a hand-written rubric it adds requirements such as remediation advice, and correct passed falls from 97\% to 6\%. From a rubric that human labels had taught to fail any report omitting a real risk, it deletes that coverage check as ``vague'', and the catch rate on unanticipated omissions falls from 100\% to 8 to 25\%. Holding the check count fixed does not prevent the drift; only delete-only revision was safe. The effect appears with Claude Opus, Sonnet and Haiku as revisers and on both human-labeled datasets. Attribution exposes the mechanism: of 186 checks that self-review deleted, 78 held a catch no other check made and 2 carried any false-fail blame.

Fourth, rubric revisions can be gated like code changes (Section~\ref{sec:gates}). We compare a regression test on labeled calibration data, a provenance freeze with a label-free screen, a Shapley gate, and their combination. Only the combination protects both failure directions: on held-out environments it keeps 9 of 10 drifting chains within 5 points of their starting quality, with decisions unchanged across three judge models, and on RAGTruth it holds the label-taught rubric at its starting quality.

\paragraph{Contributions.}
\begin{enumerate}
\item A controlled decomposition of judge strictness into rubric content, rendering, aggregation rule and judge model, showing that content dominates, with a prevalence analysis of why accuracy conceals it.
\item An exact, closed-form Shapley attribution for all-must-pass rubrics that yields per-check catch credit, false-fail blame and unique catches in $O(nk)$ time.
\item Evidence, across three reviser models, three revision constraints and three datasets, that model self-review of a rubric drifts toward the reviser's own standard, together with an attribution-based account of what it removes.
\item Change-control gates for rubric revision, evaluated with leave-one-out labels on 12 synthetic and 8 human-labeled revision chains and under three judge models.
\end{enumerate}

\section{Related work}
\label{sec:related}

\paragraph{LLM judges and their biases.} LLM-as-a-judge was popularized by MT-Bench and Chatbot Arena \citep{zheng2023judging}. Documented failure modes include position bias \citep{wang2023fair}, self-preference \citep{panickssery2024llm} and correlated errors across models \citep{kim2025correlated}; panels of diverse judges reduce some of them \citep{verga2024replacing}, and fine-tuned evaluators such as Prometheus condition on a rubric \citep{kim2023prometheus}. JudgeBench measures judge accuracy \citep{tan2024judgebench}, and the alternative annotator test asks when a judge may replace a human annotator \citep{calderon2025alternative}. This literature holds the rubric fixed and varies the judge; we do the reverse and find the rubric to be the larger factor. \citet{bagaria2026judging} report that rubric text alone predicts judge labels on HealthBench and ResearchRubrics, consistent with the rubric carrying the standard.

\paragraph{Checklists and generated rubrics.} TICK \citep{cook2024tick} and CheckEval \citep{lee2024checkeval} decompose evaluation into generated yes/no checklists; HealthBench uses physician-written criteria \citep{arora2025healthbench}. Rubrics and checklists serve as rewards \citep{gunjal2025rubrics,viswanathan2025checklists}, and recursive decomposition with filtering of misaligned criteria improves rubric-based judges \citep{shen2026rethinking}. Recent meta-evaluations find generated rubrics coherent within a model family but poorly transferable \citep{siro2026learning}, and longer, finer-grained and skewed toward high scores than human-written ones \citep{hong2026can}; RIFT offers a taxonomy of rubric failure modes with automated diagnostics \citep{qi2026rift}. Our contribution is causal: a comparison of hand-written and generated checks under fixed judges and outputs that isolates task-irrelevant requirements as the source of strictness.

\paragraph{Rubric refinement and drift.} EvalGen observed criteria drift, in which users' criteria shift as they grade, and proposed aligning assertions with human grades \citep{shankar2024validates}; reflect-and-revise refines essay-scoring rubrics from disagreements with human scores \citep{harada2025automated}. Closest to our work, \citet{ding2026rubrics} show that criterion-preserving rubric edits that pass standard validation can shift a judge's preferences on target domains by up to 27.9\% and that the shift propagates through alignment; Rulers prevents per-instance reinterpretation by freezing a rubric into a hashed specification \citep{hong2026rulers}. We study a different source of drift, model self-review without labels, characterize its direction, and contribute per-check attribution and gates that decide which edits to accept rather than freezing the whole rubric. That self-review fails without an external signal echoes results on self-correction in reasoning \citep{huang2024large} and contrasts with the gains reported for self-refinement of outputs \citep{madaan2023selfrefine}.

\paragraph{Rubrics as rewards.} Optimizing against a proxy reward overfits it \citep{gao2023scaling}. Rubric-based reinforcement learning can be exploited through partially satisfied compound criteria \citep{mahmoud2026reward}, and generated rubrics can reward dishonest answers on unanswerable tasks \citep{xie2026impossiblerubrics}; Constitutional AI relies on model critique against written principles \citep{bai2022constitutional}. Our drift result implies that a rubric revised by the model family it grades can absorb that family's own standard without any optimization pressure.

\paragraph{Shapley attribution.} Shapley values \citep{shapley1953value} underlie feature attribution \citep{lundberg2017unified} and data valuation \citep{ghorbani2019data}, where they are usually approximated by sampling. For the all-must-pass rubric game the value is exact and closed-form, which makes per-check attribution cheap enough to run inside a revision loop.

\section{Experimental setup}
\label{sec:setup}

\subsection{Tasks and datasets}

\paragraph{Recon.} We built 12 synthetic cloud environments (storage exposure, legacy authentication, dangling Domain Name System (DNS) records, guest sharing, open management ports, leaked secrets and others). Each has a ground-truth inventory with two planted findings and several minor, optional observations. \emph{Recon-72}, used for the strictness experiments, holds 2 correct reports and 4 single-flaw reports per environment (a missed finding, a fabricated finding, a finding on the wrong resource, an overstated impact): 24 correct, 48 flawed. \emph{Recon-108}, used for all revision and gating experiments, adds per environment three reports judged against an extended inventory that contains one additional real exposure outside the planted scenario: a correct report that mentions it, a report with a fabricated finding, and an \emph{unplanned miss}, the original correct report, which now omits a real exposure. Rubric authors and revisers never saw the extended inventories. Unplanned misses test whether a rubric can fail a report for omitting something its author did not anticipate. Recon-108 holds 36 correct and 72 flawed reports. For gating, 4 environments (36 reports) serve as calibration and 8 (24 correct, 48 flawed, 8 unplanned misses) as evaluation.

\paragraph{SummEdits.} From the chat and meeting domains of SummEdits \citep{laban2023llms} we sampled 146 summaries of 38 conversations, up to two human-labeled consistent and two inconsistent summaries per conversation (73 consistent). Twelve conversations (46 items, 22 consistent) form the calibration split and the remaining 100 items the evaluation split.

\paragraph{RAGTruth.} From the RAGTruth test split \citep{wu2024ragtruth} we drew 40 sources (20 question answering with retrieved passages, 20 news summarization, each at most 3{,}000 characters) and, for each, two responses with no annotated hallucination span and two with at least one, keeping only responses the dataset marks as good quality: 160 responses, 80 faithful. Twelve sources (48 responses, 24 faithful) form the calibration split and 28 sources (112 responses, 56 faithful) the evaluation split. The responses were written by the LLMs used to build RAGTruth, not by the models we study.

\subsection{Labels}
\label{sec:labels}

Every label in this paper is a single bit per output: pass or fail. There are no per-check labels, no scores and no written rationales. On Recon the bit is known by construction (a report passes iff it is correct and complete); on SummEdits it is the human judgment consistent or inconsistent; on RAGTruth it is faithful or containing a hallucination, and although RAGTruth also marks the hallucinated span, we use only the binary label. Labels are used in two ways, always from the calibration split (36 reports on Recon, 46 summaries on SummEdits, 48 responses on RAGTruth). To \emph{teach} a rubric, the reviser sees each calibration output with its label next to the judge's verdict and per-check answers under the current rubric, with disagreements flagged, for example:
\begin{quote}\small
Report: ``Critical: storage account datalake01, container `hr-exports', has anonymous read enabled and holds payroll CSVs; \ldots{} Medium: diagnostic logging is disabled on datalake01, \ldots''\\
Human label: pass; judge with your current rubric: fail \texttt{<-- DISAGREES}; judge's answers: [Yes, Yes, Yes, Yes, No, No, No, Yes]
\end{quote}
This correct report is failed by three checks of the generated draft that ask for things it need not contain (logging on a second account, resource locks, and a recommendation to disable anonymous access); disagreements like this one are what taught the label-taught rubric that remediation and these minor items are optional. To \emph{gate} a revision, the same labels are combined with the judge's per-check answers: one pass/fail bit per output is enough to tell whether each check that answered No caught a flaw or failed a correct output, which is all the attribution of Section~\ref{sec:shapley} requires. On Recon the reviser was told the labels came from security analysts; the full teaching prompts and labeled examples from both public datasets are in Appendices~\ref{app:prompts}, \ref{app:summedits} and~\ref{app:ragtruth}.

\subsection{Judging protocol and models}

The judge sees the source (inventory, conversation, or passages), one output and the rubric, and answers each check Yes or No. An output passes if every check is Yes (all-must-pass) unless stated otherwise. Each judging batch contains at most one output per source, so the judge never compares versions of the same output, and item identifiers are opaque. Judges and revisers run as sandboxed sessions that read only their input file and write only their output file. We use Claude Haiku 5.5, Sonnet 5.5 and Opus 5.5; Sonnet is the default judge and Opus the default rubric author and reviser.

\subsection{Metrics}

We report two rates for every rubric: \emph{correct passed}, the share of correct outputs that pass, and \emph{flawed caught}, the share of flawed outputs that fail. On Recon-108 we also report \emph{unplanned caught} on the unplanned-miss slice. We avoid accuracy for the reason given in Section~\ref{sec:prev}. Paired comparisons between rubric versions use an exact McNemar test on the outputs whose verdict changed.

\subsection{Rubrics}

On Recon, the \emph{hand-written} rubric has 6 checks per environment written by the author against the planted scenario. Two rubrics were written by Opus from the inventories alone: a \emph{ranked} rubric of exactly 12 checks per environment, ordered by importance so that its prefixes give the 4- and 8-check rubrics of Section~\ref{sec:strict}, and a \emph{generated draft} of 4 to 10 checks per environment (92 in all, 7.7 per environment) that is the starting point of every revision chain from a generated rubric. The \emph{label-taught} rubric is the generated draft revised once by Opus after seeing the 36 labeled calibration reports with the judge's verdicts; it learned that remediation advice is optional and added an open coverage check (``does the report mention every other real weakness in the inventory?''), averaging 7.9 checks per environment. On SummEdits and RAGTruth, Opus wrote one ranked list of 10 generic checks per dataset from the same prompt template, and a label-taught version was produced in the same way from each calibration split. Appendix~\ref{app:recon} shows one environment with all of its reports and rubric versions, and Appendix~\ref{app:prompts} reproduces every prompt.

\section{Rubric content sets strictness}
\label{sec:strict}

\subsection{Decomposition on Recon-72}

We varied the rubric along four axes: \emph{content} (a free-form verdict, a generic 1 to 5 scale over groundedness, completeness and actionability, the hand-written checks, the generated checks), \emph{count} (nested prefixes of 2, 4 and 6 hand checks and of 4, 8 and 12 generated checks), \emph{rendering} (a yes/no checklist, one paragraph read holistically, or each check scored 1 to 5 with pass at 4 or above) and \emph{judge model}. Sonnet conditions ran twice; Haiku and Opus once.

\begin{figure}[t]
\centering
\includegraphics[width=\linewidth]{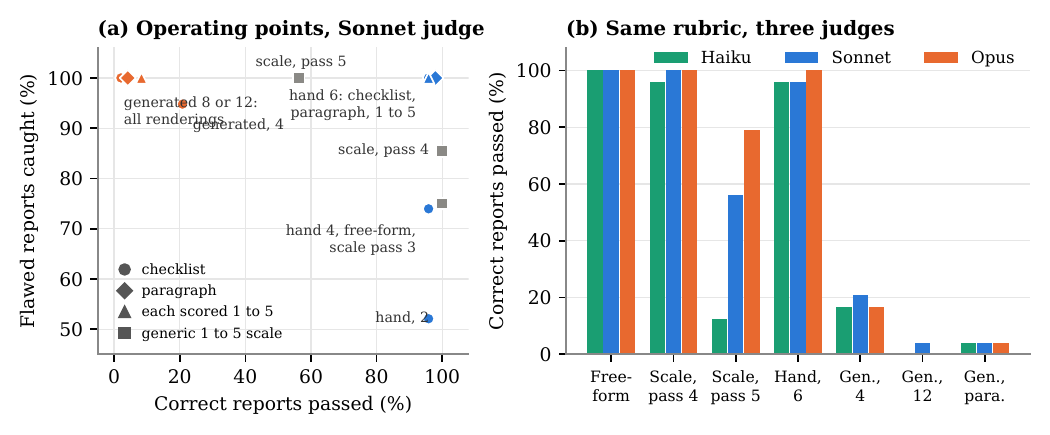}
\caption{Strictness on Recon-72 (24 correct, 48 flawed reports). (a) Operating point of each rubric under a Sonnet judge. Hand-written and generated checks form two clusters that rendering does not move. (b) Correct reports passed under the same rubrics by three judges, which agree within a few points except with a generic scale and a pass threshold of 5. \emph{Takeaway:} what the checks ask for moves correct passed from about 4\% to about 96\%, while rendering and judge model move it by a few points, as claimed: strictness is set by rubric content.}
\label{fig:strict}
\end{figure}

\paragraph{Content decides; rendering and judge do not.} Figure~\ref{fig:strict}a shows two clusters. The six hand checks pass 95.8\% of correct reports as a checklist, 97.9\% as a paragraph and 95.8\% when each is scored 1 to 5, and catch every flaw in all three renderings. The twelve ranked generated checks pass 4.2\%, 4.2\% and 8.3\%, also catching every flaw. Across judges (Figure~\ref{fig:strict}b) the hand checklist passes 95.8 to 100\% of correct reports and the generated checklist 0 to 4.2\%. Nearly every generated check that fails a correct report asks for content outside the output contract, such as ``does the report recommend disabling anonymous access?''.

\paragraph{Count costs only what the added checks cost.} Nested hand checks raise flawed caught from 52.1\% (2 checks) to 74.0\% (4) to 100\% (6) at a constant 95.8\% correct passed. Nested prefixes of the ranked generated rubric pass 20.8\% (4 checks), 2.1\% (8) and 4.2\% (12). Under all-must-pass, correct passed is the probability that every check answers Yes; if false fails were independent it would equal the product of the per-check pass rates on correct outputs. Across 12 rubric versions on Recon-108 the per-environment product predicts the observed rate within 2 points. Ten checks that each wrongly fail 5\% of correct outputs pass only $0.95^{10}\approx 60\%$ of them, and no single check looks wrong on inspection.

\paragraph{A scale threshold is a per-judge dial.} The one condition on which judges diverge is the generic 1 to 5 scale with pass only at 5: Haiku passes 12.5\% of correct reports, Sonnet 56.2\% and Opus 79.2\%, because the threshold measures how readily each model awards a top score. On reports with a missed finding, Sonnet's pass rate is 95.8\% at threshold 3, 58.3\% at 4 and 0\% at 5.

\paragraph{Leniency belongs to the empty rubric.} With a free-form verdict every judge passes every correct report, and Sonnet also passes 100\% of reports with a missed finding (Haiku 83.3\%, Opus 75.0\%). Nothing asked the judge to check completeness.

\paragraph{The aggregator outweighs the wording.} With per-check answers fixed, changing the generated draft's aggregator on Recon-108 from all-must-pass to ``at most one No'' raises correct passed from 5.6\% to 69.4\%. Three rounds of rewording by self-review left it at 5.6\%.

\subsection{Human-labeled datasets}
\label{sec:strict-human}

\begin{figure}[t]
\centering
\includegraphics[width=\linewidth]{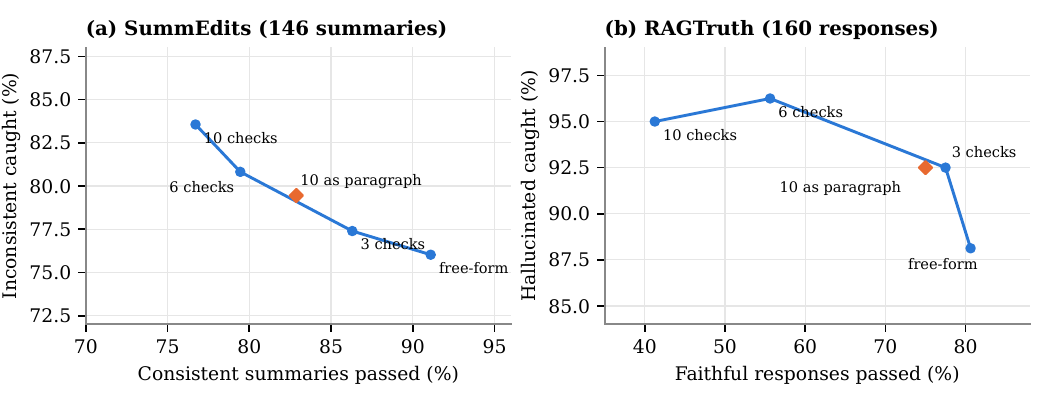}
\caption{Generic checklists of increasing length under a Sonnet judge, mean of two runs. (a) On SummEdits each added check trades a few consistent summaries passed for inconsistent ones caught. (b) On RAGTruth, going from 3 to 10 checks costs 36 points of faithful responses passed for at most 4 points of catches; the same 10 checks read as a paragraph pass 75\%. \emph{Takeaway:} the content effect replicates on human labels, small where generic checks rarely demand anything a correct output lacks (SummEdits) and large where they forbid the light synthesis annotators accepted (RAGTruth); RAGTruth is also the one place where rendering mattered, through the aggregator.}
\label{fig:human}
\end{figure}

The same count experiment with generic checklists reproduces the effect on human labels (Figure~\ref{fig:human}). On SummEdits, free-form passed 91.1\% of consistent summaries and caught 76.0\% of inconsistent ones; 3, 6 and 10 checks moved this to 86.3/77.4, 79.5/80.8 and 76.7/83.6. The direction matches Recon with a smaller magnitude, since generic checks rarely demand content a consistent summary lacks.

On RAGTruth the effect appears at full size. Free-form passed 80.6\% of faithful responses and caught 88.1\% of hallucinated ones; 3, 6 and 10 checks gave 77.5/92.5, 55.6/96.2 and 41.2/95.0. Ten checks fail 29 and 34 faithful responses (two runs) that free-form passes and pass none that free-form fails ($p<10^{-8}$ in each run), for 7 points of additional catches. The strictness is heavier in question answering, where 33.8\% of faithful answers pass ten checks against 48.8\% of faithful summaries: the check against ``conclusions, inferences, causal links, or generalizations that go beyond what the source directly supports'' fails 21 to 22 of the 40 faithful answers, the check against merging passages 15 to 17, and the check against unsupported answers 11 to 14. Answering a question from several passages requires exactly the light synthesis these checks forbid, and the annotators accepted it. Here rendering matters: the same ten checks read as one paragraph pass 75.0\% of faithful responses. Since the paragraph carries the same propositions, the difference lies in the aggregator; read holistically, the judge does not let one borderline inference fail an otherwise faithful answer.

\subsection{Why accuracy hides strictness}
\label{sec:prev}

\begin{figure}[t]
\centering
\includegraphics[width=0.52\linewidth]{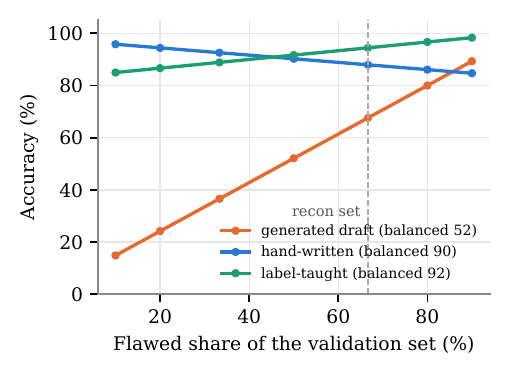}
\caption{Accuracy of three Recon rubrics as a function of the flawed share of the validation set, computed from their two rates. The generated draft has balanced accuracy 52, the level of chance, yet outscores the hand-written rubric whenever more than 80\% of the set is flawed. \emph{Takeaway:} accuracy on a defect-heavy validation set can rank a rubric that fails almost everything above one that works, which is why we report the two rates throughout.}
\label{fig:prev}
\end{figure}

Judge validation sets are commonly built by planting defects, so most of their items are flawed, and accuracy is a prevalence-weighted mixture of the two rates. The generated draft passes 5.6\% of correct reports and catches 98.6\% of flawed ones; its accuracy is 89\% on a set that is 90\% flawed and 52\% on a balanced set (Figure~\ref{fig:prev}). Reporting both rates, or balanced accuracy, removes the dependence.

\section{Exact Shapley attribution for all-must-pass rubrics}
\label{sec:shapley}

Which checks are responsible for a rubric's catches and for its false fails? Per-check fail rates do not answer this, because a flawed output failed by three checks should not count three times. We treat each graded output as a cooperative game among the checks.

Fix an output $i$ and the rubric's checks $K$, and let $F_i \subseteq K$ be the checks the judge answered No. Define $v_i(S) = \mathbb{1}[S \cap F_i \neq \emptyset]$: a coalition of checks fails the output if any of its members does, and the rubric's verdict is fail iff $v_i(K)=1$.

\begin{proposition}
For the game $v_i$, the Shapley value of check $j$ is $\phi_j(v_i) = 1/|F_i|$ if $j \in F_i$ and $0$ otherwise.
\end{proposition}
\begin{proof}
If $j \notin F_i$, $j$ never changes $v_i$ and is a null player, so $\phi_j=0$. If $j \in F_i$, in a uniformly random ordering of $K$ the marginal contribution of $j$ is 1 exactly when $j$ precedes every other member of $F_i$, which happens with probability $1/|F_i|$, and 0 otherwise. Efficiency holds: $\sum_j \phi_j = 1 = v_i(K)$.
\end{proof}

Summing over labeled outputs gives each check a \emph{catch credit} $C_j = \sum_{i\,\text{flawed}} \phi_j(v_i)$ and a \emph{false-fail blame} $B_j = \sum_{i\,\text{correct}} \phi_j(v_i)$. A check holds a \emph{unique catch} on flawed output $i$ when $F_i = \{j\}$: removing it would let that output pass. By efficiency, credits sum to the number of flawed outputs caught and blames to the number of correct outputs failed, so the two headline rates decompose exactly over checks. Computing all three quantities takes one pass over the answer matrix, $O(nk)$ for $n$ outputs and $k$ checks, against $2^k$ coalitions for a generic Shapley computation. (The Banzhaf value would assign $2^{1-|F_i|}$ and lose efficiency.) Attribution uses the per-check answers that checklist judges already produce, and assumes an answer does not depend on which other checks are present; in the nested rubrics of Section~\ref{sec:strict}, the same check graded inside a shorter or longer checklist changed 2.1 to 3.5\% of answers, within the 0 to 3.2\% range between plain reruns of one rubric. Appendix~\ref{app:shapley-example} works through the attribution of one environment.

\paragraph{What is perturbed.} Readers who know Shapley values from feature attribution may expect the output to be perturbed. Here the output and the judge's per-check answers are fixed, and the players are the checks: a coalition $S$ is a smaller rubric, and $v_i(S)$ is the verdict that rubric would give. Because the judge answers each check separately, every sub-rubric's verdict can be read off the stored answers without re-grading.

\paragraph{A worked example.} Consider a rubric of four checks graded on five outputs (Table~\ref{tab:toy}). Output R2 is failed by A and B: if the four checks are added to the rubric in a random order, the verdict flips when the first of A or B arrives, each is first in half of the orders, so each receives $\tfrac12$ and C and D receive 0.

\begin{table}[h]
\centering
\small
\begin{tabular}{llll}
\toprule
Output & Label & Checks answering No ($F_i$) & Share of the outcome \\
\midrule
R1 & flawed & A & A: 1 (unique catch) \\
R2 & flawed & A, B & A, B: $\tfrac12$ each \\
R3 & flawed & B, C, D & B, C, D: $\tfrac13$ each \\
R4 & correct & D & D: 1 of blame \\
R5 & correct & none & none \\
\bottomrule
\end{tabular}
\caption{A four-check rubric on five outputs. Totals: A has credit 1.5 and one unique catch; B 0.83 and C 0.33, both without unique catches; D has credit 0.33 and blame 1. Credits sum to the 3 flawed outputs caught and blame to the 1 correct output failed. \emph{Takeaway:} only deleting A would lose a catch and only deleting D would let a correct output pass, so unique catches and blame, not raw fail counts, identify which checks matter; these are the rules the Shapley gate applies in Section~\ref{sec:gates}.}
\label{tab:toy}
\end{table}

The totals answer the questions a reviser's edits raise, which can be checked by deleting each check and re-reading the verdicts. Deleting A lets R1 pass: A holds the only catch on R1 and must stay. Deleting B or C changes no verdict, because every output they fail is also failed by another check; credit that is shared does not protect a check. Deleting D loses no catch, since R3 is still failed by B and C, and lets the correct output R4 pass: D's blame is the cost it imposes, and a new check with blame should not be admitted. These are exactly the rules of the Shapley gate in Section~\ref{sec:gates}.

Two caveats follow from the example. First, a check's quantities depend on the other checks present: B and C are each safe to delete alone, but deleting both leaves R3 caught only by D, so attribution has to be recomputed after every accepted change. Our gates recompute between revision rounds, not between deletions within a round. Second, zero credit means only that no labeled output exercised the check. A check guarding against a flaw absent from the labeled data scores zero however necessary it is; Section~\ref{sec:gates} shows this failure for a coverage check whose flaw had a single labeled example, which is why the strongest gate pairs attribution with a provenance freeze.

Attribution immediately characterizes the rubrics of Section~\ref{sec:setup}. In the hand-written rubric, the 72 checks carry 60.0 reports of catch credit and 1.0 report of blame on Recon-108. In the label-taught rubric, the 23 checks that the label round introduced carry 25.5 reports of credit, zero blame and 13 of the rubric's 24 unique catches, while the 72 checks inherited from the generated draft carry all 6.0 reports of blame. In the generated draft itself, 92 checks carry 34.0 reports of blame among 36 correct reports.

\section{Self-review drifts toward the reviser's standard}
\label{sec:drift}

\subsection{Protocol}

Starting from the hand-written rubric (H) and the label-taught rubric (L) on Recon, a reviser revised the checks of all 12 environments for four rounds under a generic instruction: review for ambiguity, gaps, checks a different judge might misread, and flaws no check would catch. The arms were free revision with Opus (two seeds), Sonnet and Haiku; Opus with the check count held fixed; and Opus restricted to deleting or merging checks. After rounds 2 and 4 the Sonnet judge graded all 108 reports. On SummEdits and RAGTruth, Opus revised the label-taught rubric for two rounds under the free, fixed-count and delete-only constraints, and the generated rubric for two free rounds.

\subsection{Results on Recon}

\begin{figure}[t]
\centering
\includegraphics[width=\linewidth]{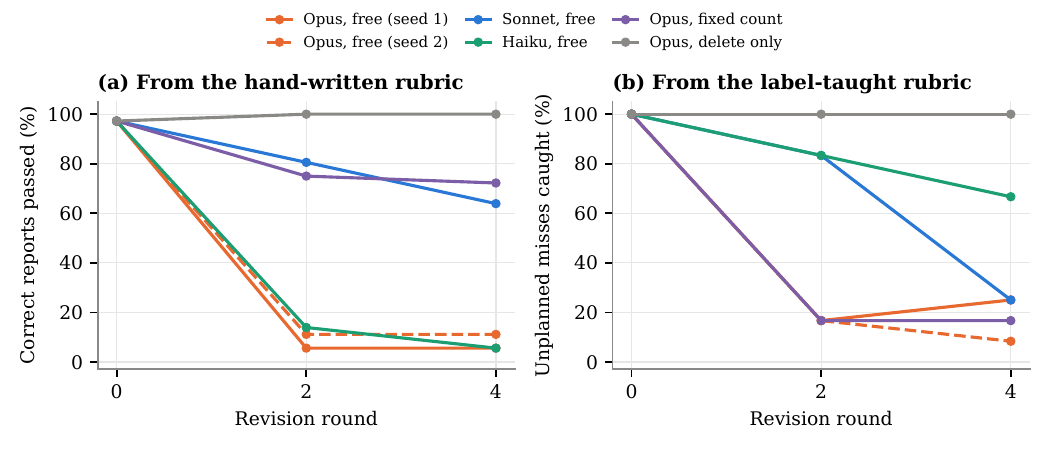}
\caption{Self-review over four rounds on Recon-108. (a) From the hand-written rubric, free revision adds requirements and correct reports stop passing. (b) From the label-taught rubric, free and fixed-count revision remove or narrow the open coverage check and unplanned misses stop being caught. Delete-only revision is flat in both. \emph{Takeaway:} drift runs in opposite directions from the two starting rubrics, too strict from one and too lenient on omissions from the other, each time toward the reviser's own idea of a good report; only the arm that cannot add requirements avoids it, consistent with the claim that self-review moves the standard rather than refining it.}
\label{fig:drift}
\end{figure}

Figure~\ref{fig:drift} shows drift in two directions. From H, free revision by Opus took correct passed from 97.2\% to 5.6\% and 11.1\% (two seeds), Haiku to 5.6\%, and Sonnet more slowly to 63.9\% (12 of 36 correct reports newly failed, none gained, $p=0.0005$); fixed-count revision reached 72.2\%. The additions were dominated by remediation requirements: Opus's round-4 rubrics contain 13 of them across the 12 environments (Haiku's 12, Sonnet's 11), against 1 at the start. The same revisions raised unplanned caught from 0\% to as much as 91.7\%, but a rubric that fails most correct reports is not usable.

From L, Opus deleted the open coverage check in nearly every environment, describing it in one change note as ``the vague completeness check''; unplanned caught fell from 100\% to 25.0\% and 8.3\%. Sonnet kept the check in place but rewrote it as a closed list of the items it expected (``apart from the payroll exposure, the logging gap\ldots''), which no longer catches an omission it did not list: 100\% to 25.0\%. Haiku kept the check but, in several environments, wrote the answer into it (``the inventory states no further weakness, so answer Yes unless the report omits the dangling record''), and unplanned caught fell to 66.7\%; fixed count fell to 16.7\%. Three revisers, three ways of closing the same check (Appendix~\ref{app:recon}). In every L arm the lost catches were all unplanned misses; the 60 planted-flaw catches were untouched. Only delete-only revision left every metric at its starting value on both rubrics (H: 97.2\% to 100\% correct passed; L: unchanged) while shrinking them from 6.0 to 4.0 and from 7.9 to 5.8 checks per environment.

Check count is not the mechanism. Free revision grew the H rubric by 1.7 to 3.1 checks per environment and shrank or preserved the L rubric; both got worse. What changed was content, and in each case it moved toward what the reviser considered a good report.

\paragraph{Model-written tests do not detect the problem; labels do.} Before these loops we tried to repair the generated draft (5.6\% correct passed) with a test-writing loop: the model wrote test reports with intended verdicts and per-check answers, a judge graded them, and the model revised on disagreement. Verdicts agreed on 71 or 72 of 72 test reports in every round, and correct passed remained at 5.6\% after three rounds, as it did under plain self-review. The model wrote its intended-pass reports to its own standard, remediation advice included, so the loop had no way to see the standard was wrong. One revision from 36 labeled reports raised correct passed to 80.6 to 83.3\% while keeping every catch; self-review then removes what the labels taught.

\subsection{Results on human-labeled data}
\label{sec:drift-human}

On SummEdits the label round relaxed clauses the annotators did not enforce and added a paraphrase check (11 checks). The first self-review round reversed this; its change note says it ``tightened leniency clauses a literal judge could exploit''. On the evaluation split, consistent summaries passed fell from 78.4\% to 54.9\% after one round and 58.8\% after two (10 lost, 0 gained, $p=0.002$), and balanced accuracy from 80.0 to 72.3. Fixed count drifted as far ($p=0.012$). Delete-only was mildest but not neutral (78.4\% to 70.6\%, 4 lost, $p=0.125$), as its merges produced longer combined checks.

On RAGTruth the label round rewrote all 10 checks to target material hallucinations and added carve-outs for peripheral imprecision, paraphrase and light inference, lifting faithful responses passed on the evaluation split from 42.9\% to 71.4\% (18 gained, 2 lost, $p<0.001$) with catches nearly unchanged (91.1\% to 89.3\%). The first self-review round split the taught contradiction check into ``a strict accuracy check for every copied detail and a contradiction check without the peripheral-error exemption, since misquoted numbers or dates in minor details are still unfaithful'', a direct reversal of what the labels had shown. Faithful passed fell to 46.4\% (14 lost, 0 gained, $p=0.0001$) and stood at 48.2\% after round two; fixed count fell to 51.8\% ($p=0.003$); delete-only was flat (69.6\%, 4 lost and 3 gained). Self-review of the generated rubric grew it from 10 to 15 checks and moved it further toward failing everything (98.2\% caught, 35.7\% passed).

\subsection{What self-review removes}
\label{sec:drift-attr}

\begin{figure}[t]
\centering
\includegraphics[width=\linewidth]{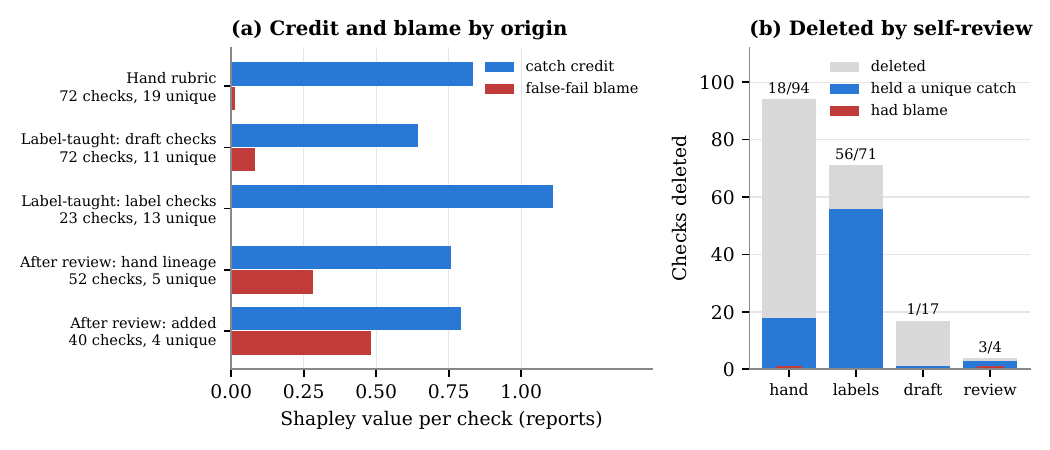}
\caption{(a) Shapley catch credit and false-fail blame per check on Recon-108, grouped by the origin of the check. Label-derived checks carry the most credit per check and no blame; checks added by Opus review of the hand rubric carry the most blame. (b) Of the 186 checks deleted across the 12 Recon revision chains, the number that held a unique catch and the number that carried any false-fail blame, by origin. \emph{Takeaway:} 78 of the 186 deleted checks held a unique catch and only 2 carried any blame, so self-review removes checks that do work no other check does and keeps the ones that fail correct outputs, the opposite of what an improving revision would do.}
\label{fig:attr}
\end{figure}

Attribution makes the mechanism visible. We trace each check's origin through revisions by text similarity to its parent. After four rounds of Opus review of the hand rubric, the 40 added checks carry 19.3 of the rubric's 34 reports of blame and rewritten hand checks another 14.7 (Figure~\ref{fig:attr}a). Figure~\ref{fig:attr}b shows what the 12 Recon chains deleted: of 186 deleted checks, 179 had caught at least one flawed report and 78 held a unique catch, among them 56 of the 71 deleted label-derived checks. Only 2 of the 186 carried any false-fail blame. Self-review removes the checks that do work no other check does and keeps, or adds, the checks that fail correct outputs.

\section{Change control for rubric revision}
\label{sec:gates}

If rubric content carries the standard and unsupervised revision moves it, a revision should be accepted the way a code change is: check by check, against evidence. Each revision proposes a new rubric; a gate decides, per environment and per check, which changes to accept. Proposed checks are matched to current ones by text similarity and treated as rewrites, or otherwise as additions; unmatched current checks are deletions. We compare four gates.
\begin{itemize}
\item \textbf{Gate R (regression).} Accept the whole revision only if, on the calibration reports, correct passed and flawed caught each drop by fewer than 2 reports and no unplanned miss is lost.
\item \textbf{Gate P (provenance, label-free).} Checks that came from a human author or from labels are frozen against deletion and rewriting; new checks that fail more than half of all reports in their environment are blocked.
\item \textbf{Gate S (Shapley).} Block deletion of any check holding a unique catch; block any addition with false-fail blame above zero; block a rewrite that adds a false fail or loses a unique catch. Labels come from the same environment with the scored report held out (leave-one-out), so no report's own label decides its verdict.
\item \textbf{Gate S+P.} Gate S together with the provenance freeze of Gate P.
\end{itemize}
Appendix~\ref{app:gate-example} traces Gates S and S+P through one environment.

\subsection{Results on Recon}

\begin{figure}[t]
\centering
\includegraphics[width=\linewidth]{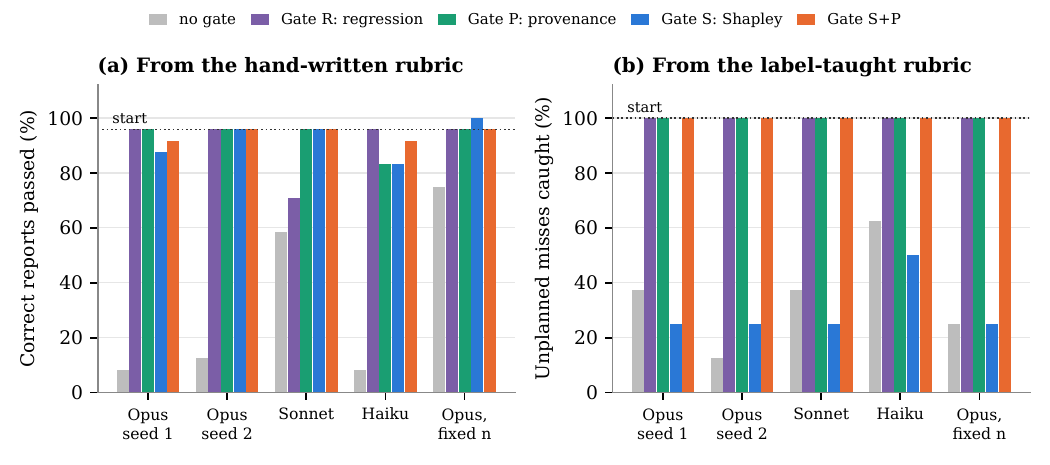}
\caption{Quality of the gated rubric after four rounds of revision, scored on the 8 evaluation environments. (a) From the hand-written rubric the failure mode is strictness; (b) from the label-taught rubric it is lost coverage. Only Gate S+P stays near the starting level (dotted) in both panels. \emph{Takeaway:} each single signal protects one failure direction (Gate R only what its calibration labels see, Gate P coverage, Gate S strictness), and only their combination protects both, with one exception (Sonnet's label-taught chain) discussed in the text; this supports gating revisions check by check rather than trusting either labels or attribution alone.}
\label{fig:gates}
\end{figure}

Figure~\ref{fig:gates} summarizes the ten free and fixed-count chains; the delete-only chains did not need a gate.

Gate R rejected both rounds in 9 of 10 chains and so restored the start. It failed on Sonnet's H chain: round 2 changed no verdict in the 4 calibration environments and was accepted, although on the evaluation environments the accepted rubric passes 70.8\% of correct reports against 95.8\% at the start. A regression gate sees only what its calibration set covers.

Gate P protects coverage on every L chain (100\% unplanned caught) and blocks most remediation additions on the H chains, but on Haiku's H chain new checks that fail a minority of reports pass its label-free screen (83.3\% correct passed).

Gate S handles strictness, keeping correct passed at 83.3 to 100\% on the H chains, but cannot protect the open coverage check: each environment holds one unplanned miss, and under leave-one-out that report's label is the only witness, so the check never shows a unique catch and its deletion is accepted (unplanned caught 25 to 50\%).

Gate S+P keeps unplanned caught at 100\% on every L chain and correct passed at 91.7 to 95.8\% on every H chain, with flawed caught at its starting value. The single exception is Sonnet's L chain (58.3\% correct passed against 75.0\%): the rewritten coverage check matched no frozen check, entered as an addition, and failed exactly one correct report per environment, which leave-one-out cannot detect.

Replaying gates from stored per-check answers agrees with grading the gated rubric afresh: re-grading the S+P rubrics of Opus's H and L chains and the Gate P rubrics of both Opus L chains from scratch gave results within 4.2 points of the replay on every metric.

\subsection{Robustness to the judge}

We re-graded six rubric versions (the H and L starts and rounds 2 and 4 of one Opus chain each) with Haiku and Opus as judges. Per-check answers agreed on 97.0 to 98.1\% of 4{,}806 check decisions and verdicts on 98.6 to 99.5\% of reports. Table~\ref{tab:judges} replays the gates under each judge; Gate R rejected both rounds under every judge, and Gate S+P keeps every metric within 6 points of that judge's own starting value.

\begin{table}[t]
\centering
\small
\begin{tabular}{llccc}
\toprule
Chain & Rubric or gate & Haiku judge & Sonnet judge & Opus judge \\
\midrule
H, Opus seed 1 & start & 100.0 / 83.3 / 0.0 & 97.2 / 83.3 / 0.0 & 97.2 / 83.3 / 0.0 \\
 & ungated (round 4) & 5.6 / 98.6 / 91.7 & 5.6 / 98.6 / 91.7 & 5.6 / 98.6 / 91.7 \\
 & Gate S+P, leave-one-out & 97.2 / 83.3 / 0.0 & 94.4 / 83.3 / 0.0 & 91.7 / 83.3 / 0.0 \\
\midrule
L, Opus seed 2 & start & 80.6 / 98.6 / 91.7 & 83.3 / 100 / 100 & 83.3 / 100 / 100 \\
 & ungated (round 4) & 83.3 / 84.7 / 16.7 & 80.6 / 84.7 / 8.3 & 80.6 / 86.1 / 16.7 \\
 & Gate S+P, leave-one-out & 83.3 / 98.6 / 91.7 & 80.6 / 100 / 100 & 80.6 / 100 / 100 \\
\bottomrule
\end{tabular}
\caption{Gate replay under three judge models on all 12 Recon environments. Each cell: correct passed / flawed caught / unplanned caught (\%). \emph{Takeaway:} under every judge the ungated rubric collapses in the same way, and Gate S+P keeps each metric within 6 points of that judge's own starting value, so the gate's decisions do not depend on which model applies the rubric.}
\label{tab:judges}
\end{table}

\subsection{Results on human-labeled data}
\label{sec:gates-human}

On SummEdits and RAGTruth the gates decide with the calibration labels and are scored on the evaluation split; the freeze covers the label-round checks that differ from the generated draft (7 of 11 on SummEdits, all 10 on RAGTruth). Table~\ref{tab:human} reports the results.

\begin{table}[t]
\centering
\small
\begin{tabular}{lcccccc}
\toprule
Chain & start & ungated & Gate R & Gate P & Gate S & Gate S+P \\
\midrule
\multicolumn{7}{l}{\emph{SummEdits} (51 consistent, 49 inconsistent evaluation items)} \\
Generated, free & 72.5 / 79.6 & 58.8 / 87.8 & 64.7 / 83.7 & 58.8 / 87.8 & 68.6 / 81.6 & 68.6 / 81.6 \\
Label-taught, free & 78.4 / 81.6 & 58.8 / 85.7 & 58.8 / 85.7 & 62.7 / 85.7 & 64.7 / 85.7 & 68.6 / 85.7 \\
Label-taught, fixed count & 78.4 / 81.6 & 60.8 / 85.7 & 78.4 / 81.6 & 70.6 / 83.7 & 68.6 / 83.7 & 70.6 / 83.7 \\
Label-taught, delete only & 78.4 / 81.6 & 70.6 / 85.7 & 70.6 / 85.7 & 74.5 / 81.6 & 72.5 / 83.7 & 74.5 / 81.6 \\
\midrule
\multicolumn{7}{l}{\emph{RAGTruth} (56 faithful, 56 hallucinated evaluation items)} \\
Generated, free & 42.9 / 91.1 & 35.7 / 98.2 & 42.9 / 91.1 & 35.7 / 98.2 & 35.7 / 96.4 & 35.7 / 96.4 \\
Label-taught, free & 71.4 / 89.3 & 48.2 / 94.6 & 71.4 / 89.3 & 53.6 / 94.6 & 58.9 / 83.9 & 71.4 / 91.1 \\
Label-taught, fixed count & 71.4 / 89.3 & 51.8 / 92.9 & 71.4 / 89.3 & 53.6 / 89.3 & 85.7 / 92.9 & 71.4 / 89.3 \\
Label-taught, delete only & 71.4 / 89.3 & 69.6 / 87.5 & 69.6 / 87.5 & 71.4 / 89.3 & 78.6 / 85.7 & 71.4 / 89.3 \\
\bottomrule
\end{tabular}
\caption{Gates on the human-labeled datasets after two rounds of Opus revision, decided with calibration labels and scored on the evaluation split. Each cell: correct outputs passed / flawed outputs caught (\%). \emph{Takeaway:} on every label-taught chain Gate S+P passes more correct outputs than the ungated rubric, at a cost of at most 4.1 points of catches; it returns RAGTruth to its starting value and recovers part of the SummEdits loss, consistent with the gates protecting what their calibration labels and frozen checks can see.}
\label{tab:human}
\end{table}

The two datasets behave differently, and the difference is instructive. On RAGTruth the calibration split exposes the drift (faithful passed on calibration fell from 79.2\% to 62.5\% and 45.8\%), so Gate R rejects every damaging round, and Gate S+P, which here freezes all ten taught checks, holds the start (71.4/91.1 on the free chain; balanced accuracy 81.2 against 80.4 at the start). Gate S alone permits deletion of taught checks whose catches are shared and falls to 58.9\% on the free chain, while on the fixed-count chain it accepts only rewrites that add no false fail on calibration and finishes above the start (85.7/92.9). On SummEdits, Gate R accepted the second free round because calibration shows a drop of only one consistent summary: the label round was taught on these very items, and the drift concentrates elsewhere. Gate S+P recovers about half of the lost passes (68.6\%; balanced accuracy 77.2 against 72.3 ungated); the taught standard there lives in clauses inside checks that partly match the generated draft verbatim, so a check-level freeze protects it only in part.

\section{Discussion}
\label{sec:discussion}

\paragraph{What the results imply.} Three implications go beyond the individual experiments. First, when a checklist grader looks too lenient or too strict, the rubric is the first place to look and the judge model the last. Replacing the judge moved correct passed by a few points; the content of the checks moved it by up to 98. A check that fails nearly every output, correct or not, is almost always demanding something the task does not require, and it can be found from unlabeled outputs alone. Second, a rubric's standard cannot be repaired from inside the model that wrote it. Self-review imported the reviser's own standard with all three revisers and on all three datasets, and a loop in which the model wrote its own test cases agreed with itself on 71 or 72 of 72 tests while the rubric still failed 94\% of correct reports. In every setting, the only intervention that moved the standard toward the intended one was a small set of labeled outputs, and self-review then undid it. Third, the same mechanism matters wherever rubrics serve as rewards. A rubric revised by the model family it grades can absorb that family's preferences with no optimization pressure at all, so a policy trained on it is rewarded for agreeing with its own family rather than with the task. This is a quieter route to reward hacking than the exploitation studied in prior work \citep{mahmoud2026reward,xie2026impossiblerubrics}.

\paragraph{Recommendations.} For teams that grade with checklists:
\begin{itemize}
\item Report correct passed and flawed caught, or balanced accuracy, never accuracy alone; a validation set built from planted defects rewards rubrics that fail everything.
\item Treat every added check as a cost. Under all-must-pass, each check multiplies the pass rate of correct outputs by its own, so audit each check's fail rate on unlabeled outputs and keep only checks that do work no other check does.
\item Choose the aggregation rule deliberately and report it; on our data it moved correct passed by more than any rewording.
\item Record the provenance of every check and freeze human- and label-derived checks against automated revision.
\item Permit unattended revision to delete or merge, never to add or tighten. Even this is not free: on SummEdits, merges produced longer checks and cost 8 points of correct passed.
\item Gate every revision with per-check Shapley credit and blame on labeled data, recompute attribution after each accepted change, and keep at least two labeled examples of each rare flaw per environment.
\end{itemize}

\paragraph{What each gate protects, and where it fails.} No single gate was sufficient, and the failures are informative. The regression gate (R) catches any damage its calibration labels can see and nothing else: it accepted a revision that changed no calibration verdict but cut correct passed elsewhere from 95.8\% to 70.8\%, and on SummEdits it missed drift concentrated outside the items the rubric had been taught on. The provenance gate (P) protects label-derived checks without any labels at revision time but lets harmful new checks through when they fail only a minority of outputs. The Shapley gate (S) blocks harmful additions and rewrites precisely, but cannot protect a check whose flaw is rare in the labeled data, because a zero score there means no evidence rather than no value. The combination (S+P) covers both directions and was the most reliable, yet it inherits two limits: a standard expressed as a clause inside a check is protected only as far as the check-level freeze reaches, and a rewritten check that the similarity matcher fails to recognize enters as an addition. In practice the gates are complements: S+P for routine revisions, R as a final check on a calibration set refreshed with new labels.

\paragraph{Limitations.} The Recon environments and labels are synthetic and ours; the SummEdits and RAGTruth replications cover consistency and faithfulness grading with human labels but not agentic tasks. All judges, rubric authors and revisers are Claude models (the RAGTruth responses come from other LLMs), so the insensitivity to the judge is shown within one model family, and the direction of drift toward remediation advice may be particular to these revisers' standards even though drift itself appeared with all three. Samples are small: 24 to 36 correct reports per comparison and 8 to 12 unplanned misses, so differences of a few points are within noise, while the main effects (97\% to 6\%, 100\% to 8\%) lie far outside it. Free revision on Recon ran with two Opus seeds and one seed for the other arms, and each rubric version was graded once per round. Gates were evaluated by replay from stored per-check answers, which assumes a check's answers do not depend on its neighbors; spot re-grades agreed within 4.2 points. Check matching across revisions relies on a text-similarity threshold, which misclassified one rewrite as an addition in the failure case of Section~\ref{sec:gates}.

\section{Conclusion}

A checklist rubric carries the standard its judge enforces: content sets strictness, the judge model and the rendering move it by a few points, and the common practice of improving a rubric by model self-review moves the standard toward the reviser's own, deleting exactly the checks that labels taught. Because all-must-pass rubrics admit exact Shapley attribution in closed form, each check's contribution can be measured at negligible cost, and revisions can be gated like code changes. A gate that combines Shapley rules with a provenance freeze held nearly every drifting chain at its starting quality, on synthetic and human-labeled data and under three judge models.

\appendix
\newenvironment{promptbox}{\par\medskip\begin{quote}\small\ttfamily\raggedright\sloppy}{\end{quote}\medskip}
\newenvironment{checks}{\begin{enumerate}[leftmargin=2em,itemsep=1pt,topsep=2pt]\small}{\end{enumerate}}
\newenvironment{reports}{\begin{itemize}[leftmargin=1.2em,itemsep=2pt,topsep=2pt]\small}{\end{itemize}}

\section{Prompts}\label{app:prompts}

All prompts below are reproduced verbatim from the code that generated them (\texttt{recon/pipeline.py}, \texttt{recon/form.py}, \texttt{recon/loops.py}, \texttt{summedits/summedits.py}, \texttt{ragtruth/ragtruth.py}). Each judging or revision job was run as an isolated model session that received one instruction: ``Read the file \textless{}input\textgreater{} and follow its instructions exactly. Write only the requested JSON to \textless{}output\textgreater{}. Do not open, list or read any other file, and do not run any commands.'' Judges therefore saw the input file and nothing else.

\subsection{Judge prompt (Recon)}

The head of every checklist batch, followed by one rendered item (a correct report for the storage environment under the hand-written rubric). A batch holds one report per environment, nine batches cover Recon-108.

\begin{promptbox}
You are reviewing findings reports written by an automated security reconnaissance agent. For each item you get the environment inventory (the ground truth about the tenant), the agent's report, and a numbered rubric. Judge the report only against the inventory, and judge each item on its own. For each numbered check, answer Yes if the report satisfies what the check asks, otherwise No.\\{}
~\\{}
Reply with only a JSON array, one object per item in the order given, with one answer per check in order, like this:\\{}
[\{"id": "XXXXX", "answers": ["Yes", "No", "Yes"]\}]\\{}
~\\{}
Items:\\{}
~\\{}
Item 0IBYY\\{}
Environment: Cloud storage and data exposure\\{}
Inventory:\\{}
  - Subscription: contoso-prod. Storage accounts: datalake01, backups02, webassets03.\\{}
  - datalake01 container 'hr-exports' has anonymous (public) read enabled; holds payroll CSV files.\\{}
  - backups02: all containers private; firewall set to selected networks only.\\{}
  - webassets03 container 'cdn' is public read by design (static site assets, no sensitive data).\\{}
  - No storage account has a resource lock; soft-delete is enabled on datalake01 only.\\{}
  - Diagnostic logging is off on datalake01 and backups02.\\{}
Report: "Critical: storage account datalake01, container 'hr-exports', has anonymous read enabled and holds payroll CSVs; this exposes sensitive HR data to the public internet. Medium: diagnostic logging is disabled on datalake01, so reads of the exposed container are not recorded. webassets03 'cdn' is public but holds only static assets and is expected."\\{}
Checks:\\{}
  1. Does the report identify datalake01 'hr-exports' as a public data exposure holding sensitive payroll data?\\{}
  2. Does the report correctly place the public payroll exposure on datalake01 (not another storage account)?\\{}
  3. Does the report avoid claiming any storage exposure that is not present in the environment?\\{}
  4. Does the report treat the public 'cdn' container on webassets03 as expected rather than a finding?\\{}
  5. Does the report avoid asserting a pivot or attacker capability not supported by the inventory?\\{}
  6. Does the report flag that diagnostic logging is disabled on datalake01?\\{}
~\\{}
[further items]
\end{promptbox}

The free-form, generic-scale, paragraph and per-check-scale renderings of Section~\ref{sec:strict} replace the head and the check list as follows.

\noindent\textbf{Free-form verdict}
\begin{promptbox}
You are grading findings reports written by an automated security reconnaissance agent, which inspects a tenant's configuration and reports security posture issues. For each item you get the environment inventory (the ground truth about the tenant) and the agent's report. Judge the report only against the inventory, and judge each item on its own.\\{}
~\\{}
Answer "Pass" if the report is a correct findings report for this environment, otherwise "Fail".\\{}
~\\{}
Reply with only a JSON array, one object per item in order, like this:\\{}
[\{"id": "XXXXX", "verdict": "Pass"\}]
\end{promptbox}

\noindent\textbf{Generic 1 to 5 scale}
\begin{promptbox}
You are grading findings reports written by an automated security reconnaissance agent, which inspects a tenant's configuration and reports security posture issues. For each item you get the environment inventory (the ground truth about the tenant) and the agent's report. Judge the report only against the inventory, and judge each item on its own.\\{}
~\\{}
Rate the report from 1 to 5 on three dimensions, where 5 is best:\\{}
1. Groundedness: every claim is supported by the inventory.\\{}
2. Completeness: the report covers the security findings present in the inventory.\\{}
3. Actionability: the findings are specific enough to act on.\\{}
~\\{}
Reply with only a JSON array, one object per item in order, with one score per numbered point in order, like this:\\{}
[\{"id": "XXXXX", "scores": [4, 2, 5]\}]
\end{promptbox}

\noindent\textbf{Checks as one paragraph}
\begin{promptbox}
You are grading findings reports written by an automated security reconnaissance agent, which inspects a tenant's configuration and reports security posture issues. For each item you get the environment inventory (the ground truth about the tenant) and the agent's report. Judge the report only against the inventory, and judge each item on its own.\\{}
~\\{}
Each item also has a grading standard. Answer "Pass" if the report meets the standard, otherwise "Fail".\\{}
~\\{}
Reply with only a JSON array, one object per item in order, like this:\\{}
[\{"id": "XXXXX", "verdict": "Pass"\}]\\{}
[item]\\{}
Grading standard: \textless{}the environment's checks joined into one paragraph\textgreater{}
\end{promptbox}

\noindent\textbf{Each check scored 1 to 5 (pass if every score is at least 4)}
\begin{promptbox}
You are grading findings reports written by an automated security reconnaissance agent, which inspects a tenant's configuration and reports security posture issues. For each item you get the environment inventory (the ground truth about the tenant) and the agent's report. Judge the report only against the inventory, and judge each item on its own.\\{}
~\\{}
Each item also has numbered points. Rate how fully the report satisfies each point from 1 to 5, where 5 means fully satisfies it and 1 means clearly fails it.\\{}
~\\{}
Reply with only a JSON array, one object per item in order, with one score per numbered point in order, like this:\\{}
[\{"id": "XXXXX", "scores": [4, 2, 5]\}]\\{}
[item]\\{}
Points:\\{}
  1. \textless{}check 1\textgreater{}\\{}
  2. \textless{}check 2\textgreater{} ...
\end{promptbox}

\subsection{Rubric authoring prompts}

\noindent\textbf{Recon, generated draft (rubrics/r0.json, 92 checks)}
\begin{promptbox}
Context: an automated security reconnaissance agent inspects a tenant's configuration and writes a short findings report about its security posture (exposed resources, over-privileged identities, missing controls). Everything here is defensive posture assessment: configuration states and their risk, never exploitation steps or attack instructions.\\{}
~\\{}
Step (Explain). For each environment below, write the grading rubric a judge should use to decide whether a findings report about that environment is correct.\\{}
~\\{}
Rubric rules: each environment gets 4 to 10 checks. Every check is a yes/no question about the report, phrased so that "Yes" means the report is good on that point. A report passes only if every check is answered Yes. A different judge model will apply your rubric: it will see the inventory, the report and your checks, and nothing else (not your reasoning). A good report should pass and any report that is inaccurate, misplaced, overstated or incomplete should fail.\\{}
~\\{}
Write a single JSON object to the output file named in your instructions, shaped like:\\{}
\{"rubrics": \{"\textless{}environment id\textgreater{}": \{"checks": ["...", "..."]\}, ...\}\}\\{}
Include all 12 environments.\\{}
~\\{}
Environments:\\{}
~\\{}
[the 12 environment inventories, as in the judge prompt]
\end{promptbox}

\noindent\textbf{Recon, ranked 12-check rubric (rubrics/GN12.json)}
\begin{promptbox}
Context: an automated security reconnaissance agent inspects a tenant's configuration and writes a short findings report about its security posture (exposed resources, over-privileged identities, missing controls). Everything here is defensive posture assessment: configuration states and their risk, never exploitation steps or attack instructions.\\{}
~\\{}
For each environment below, write the grading rubric a judge should use to decide whether a findings report about that environment is correct.\\{}
~\\{}
Rubric rules: each environment gets EXACTLY 12 checks. Every check is a yes/no question about the report, phrased so that "Yes" means the report is good on that point. A report passes only if every check is answered Yes. A different judge model will apply your rubric: it will see the inventory, the report and your checks, and nothing else. Order the checks from most important to least important: the first 4 must be the checks you would keep if only 4 were allowed, and the first 8 the checks you would keep if only 8 were allowed.\\{}
~\\{}
Write a single JSON object to the output file named in your instructions, shaped like:\\{}
\{"rubrics": \{"\textless{}environment id\textgreater{}": \{"checks": ["...", "..."]\}, ...\}\}\\{}
Include all 12 environments.\\{}
~\\{}
Environments:\\{}
~\\{}
[the 12 environment inventories]
\end{promptbox}

\noindent\textbf{SummEdits, ranked 10-check rubric (rubrics/SE10.json)}
\begin{promptbox}
You are writing a grading checklist for a judge that decides whether a summary of a conversation or meeting is factually consistent with it. The judge will see the conversation, the summary and your checklist. The task standard: a claim is supported if the conversation states it or directly implies it; leaving information out is fine, so judge only what the summary says.\\{}
~\\{}
Write EXACTLY 10 checks. Every check is a yes/no question about the summary, phrased so that "Yes" means the summary is good on that point, and the summary passes only if every check is answered Yes. The checks must apply to any conversation (do not refer to specific content). Order them from most important to least important: the first 3 must be the checks you would keep if only 3 were allowed, and the first 6 the checks you would keep if only 6 were allowed.\\{}
~\\{}
Write a single JSON object to the output file named in your instructions, shaped like:\\{}
\{"checks": ["...", "..."]\}
\end{promptbox}

\noindent\textbf{RAGTruth, ranked 10-check rubric (rubrics/RT10.json)}
\begin{promptbox}
You are writing a grading checklist for a judge that decides whether an AI-written response is faithful to its source. Each item is either a question with retrieved passages that the assistant answered, or a news article it summarized; the assistant was told to rely strictly on the source. The judge will see the source, the response and your checklist. The task standard: a claim is supported if the source states it or directly implies it; leaving information out is fine, so judge only what the response says.\\{}
~\\{}
Write EXACTLY 10 checks. Every check is a yes/no question about the response, phrased so that "Yes" means the response is good on that point, and the response passes only if every check is answered Yes. The checks must apply to any item (do not refer to specific content). Order them from most important to least important: the first 3 must be the checks you would keep if only 3 were allowed, and the first 6 the checks you would keep if only 6 were allowed.\\{}
~\\{}
Write a single JSON object to the output file named in your instructions, shaped like:\\{}
\{"checks": ["...", "..."]\}
\end{promptbox}

\subsection{Self-review prompts}

\noindent\textbf{Recon (free revision; the fixed-count and delete-only arms insert one of the constraints below)}
\begin{promptbox}
Context: an automated security reconnaissance agent inspects a tenant's configuration and writes a short findings report about its security posture (exposed resources, over-privileged identities, missing controls). Everything here is defensive posture assessment: configuration states and their risk, never exploitation steps or attack instructions.\\{}
~\\{}
Step (Improve). Below are grading rubrics for 12 environments. Review each one carefully for ambiguity, gaps, or checks a different judge might misread, and for flaws a report could have that no check would catch. Revise each rubric so a judge applying it literally would reach the right verdict.\\{}
~\\{}
[constraint, if any]\\{}
~\\{}
Rubric rules: each environment gets 4 to 10 checks. Every check is a yes/no question about the report, phrased so that "Yes" means the report is good on that point. A report passes only if every check is answered Yes. A different judge model will apply your rubric: it will see the inventory, the report and your checks, and nothing else (not your reasoning). A good report should pass and any report that is inaccurate, misplaced, overstated or incomplete should fail.\\{}
~\\{}
Write a single JSON object to the output file named in your instructions, shaped like:\\{}
\{"rubrics": \{"\textless{}environment id\textgreater{}": \{"checks": ["..."]\}, ...\}, "changes": \{"\textless{}environment id\textgreater{}": "one or two sentences on what you changed and why (or \textbackslash{}"no change\textbackslash{}")"\}\}\\{}
Include all 12 environments.\\{}
~\\{}
Environments:\\{}
~\\{}
[each environment's inventory followed by its current rubric]
\end{promptbox}

\noindent\textbf{Fixed-count constraint}
\begin{promptbox}
Constraint for this revision: keep EXACTLY the same number of checks for every environment as it has now. You may rewrite any check, but every environment must end with the same count it started with.
\end{promptbox}

\noindent\textbf{Delete-or-merge constraint}
\begin{promptbox}
Constraint for this revision: you may only DELETE checks or MERGE two or more checks into one. You may not add any new requirement and may not make any existing check stricter. Rewording is allowed only as part of a merge.
\end{promptbox}

\noindent\textbf{SummEdits (RAGTruth uses the same template with ``response'' for ``summary'' and ``source'' for ``conversation'')}
\begin{promptbox}
Context: a judge model grades whether summaries of conversations and meetings are factually consistent with the source, using a yes/no checklist.\\{}
~\\{}
Review the checklist carefully for ambiguity, gaps, or checks a judge might misread, and for flaws a summary could have that no check would catch. Revise the checklist so a judge applying it literally would reach the right verdict.\\{}
~\\{}
[constraint, if any]\\{}
~\\{}
Checklist rules: write between 4 and 15 checks. Every check is a yes/no question about the summary, phrased so that "Yes" means the summary is good on that point. A summary passes only if every check is answered Yes. The checks must apply to any conversation (do not refer to specific content). A judge model will apply the checklist to each conversation and summary, seeing nothing else.\\{}
~\\{}
Write a single JSON object to the output file named in your instructions, shaped like:\\{}
\{"checks": ["...", "..."], "changes": "two or three sentences on what you changed and why"\}\\{}
~\\{}
Current checklist:\\{}
  1. ...
\end{promptbox}

\subsection{Label-teaching prompts}

On Recon, the label-taught rubric L1 was produced by the self-review prompt above with the following block inserted after the review instruction. The labels are known by construction (Section~\ref{sec:setup}); the reviser was told they came from security analysts. One labeled environment is shown.

\begin{promptbox}
In addition, security analysts labeled a small set of real reports for four of the environments. Their labels define the standard a report must meet. Below are those reports, the analysts' label, and the verdict a judge reached with your current rubric. Where they disagree, work out what standard the analysts are applying and apply the same standard to every environment's rubric, not only these four.\\{}
~\\{}
Environment id: stor (labeled examples)\\{}
~\\{}
  Report: "Critical: storage account datalake01, container 'hr-exports', has anonymous read enabled and holds payroll CSVs; this exposes sensitive HR data to the public internet. Medium: diagnostic logging is disabled on datalake01, so reads of the exposed container are not recorded. webassets03 'cdn' is public but holds only static assets and is expected."\\{}
    Human label: pass; judge with your current rubric: fail  \textless{}-- DISAGREES; judge's answers: ['Yes', 'Yes', 'Yes', 'Yes', 'No', 'No', 'No', 'Yes']\\{}
~\\{}
  Report: "The datalake01 storage account exposes payroll data: its 'hr-exports' container allows anonymous read, a critical data exposure. Logging on datalake01 is also turned off (medium), meaning access to that data leaves no audit trail. The public 'cdn' container on webassets03 is by design and not a finding."\\{}
    Human label: pass; judge with your current rubric: fail  \textless{}-- DISAGREES; judge's answers: ['Yes', 'Yes', 'Yes', 'Yes', 'No', 'No', 'No', 'Yes']\\{}
~\\{}
  Report: "Critical: datalake01 container 'hr-exports' has anonymous read enabled and holds payroll CSVs, exposing sensitive HR data publicly. webassets03 'cdn' is public but expected."\\{}
    Human label: fail; judge with your current rubric: fail; judge's answers: ['Yes', 'Yes', 'Yes', 'Yes', 'No', 'No', 'No', 'Yes']\\{}
~\\{}
  Report: "Critical: datalake01 'hr-exports' allows anonymous read and holds payroll CSVs (public data exposure). Critical: backups02 also has a publicly readable container exposing backup archives. Medium: diagnostic logging is disabled on datalake01."\\{}
    Human label: fail; judge with your current rubric: fail; judge's answers: ['Yes', 'Yes', 'Yes', 'No', 'No', 'No', 'No', 'No']\\{}
~\\{}
  Report: "Critical: storage account backups02 has a container with anonymous read enabled holding payroll CSVs, exposing sensitive HR data publicly. Medium: diagnostic logging is disabled on datalake01."\\{}
    Human label: fail; judge with your current rubric: fail; judge's answers: ['No', 'No', 'Yes', 'No', 'No', 'No', 'No', 'No']\\{}
~\\{}
  Report: "Critical: datalake01 'hr-exports' allows anonymous read of payroll CSVs. An external attacker can use these payroll records to authenticate to the HR application and pivot into the finance subscription. Medium: logging disabled on datalake01."\\{}
    Human label: fail; judge with your current rubric: fail; judge's answers: ['Yes', 'Yes', 'Yes', 'Yes', 'No', 'No', 'No', 'No']\\{}
~\\{}
  [This report was written for a version of the inventory with one extra line: Storage account backups02 still allows shared-key (account key) authorization, and its keys have not been rotated in 400 days.]\\{}
  Report: "Critical: storage account datalake01, container 'hr-exports', has anonymous read enabled and holds payroll CSVs; this exposes sensitive HR data to the public internet. Medium: diagnostic logging is disabled on datalake01, so reads of the exposed container are not recorded. webassets03 'cdn' is public but holds only static assets and is expected. Medium (additional): backups02 still allows shared-key authorization and its account keys have not been rotated in over a year."\\{}
    Human label: pass; judge with your current rubric: fail  \textless{}-- DISAGREES; judge's answers: ['Yes', 'Yes', 'Yes', 'Yes', 'No', 'No', 'No', 'Yes']\\{}
~\\{}
  [This report was written for a version of the inventory with one extra line: Storage account backups02 still allows shared-key (account key) authorization, and its keys have not been rotated in 400 days.]\\{}
  Report: "Critical: storage account datalake01, container 'hr-exports', has anonymous read enabled and holds payroll CSVs; this exposes sensitive HR data to the public internet. Medium: diagnostic logging is disabled on datalake01, so reads of the exposed container are not recorded. webassets03 'cdn' is public but holds only static assets and is expected."\\{}
    Human label: fail; judge with your current rubric: fail; judge's answers: ['Yes', 'Yes', 'Yes', 'Yes', 'No', 'No', 'No', 'Yes']\\{}
~\\{}
  [This report was written for a version of the inventory with one extra line: Storage account backups02 still allows shared-key (account key) authorization, and its keys have not been rotated in 400 days.]\\{}
  Report: "Critical: datalake01 'hr-exports' allows anonymous read and holds payroll CSVs (public data exposure). Critical: backups02 also has a publicly readable container exposing backup archives. Medium: diagnostic logging is disabled on datalake01."\\{}
    Human label: fail; judge with your current rubric: fail; judge's answers: ['Yes', 'Yes', 'Yes', 'No', 'No', 'No', 'No', 'No']\\{}
~\\{}
[three more labeled environments]
\end{promptbox}

\noindent\textbf{SummEdits (RAGTruth analogous, with ``faithful'' and ``contains a hallucination'')}
\begin{promptbox}
Context: a judge model grades whether summaries of conversations and meetings are factually consistent with the source, using a yes/no checklist.\\{}
~\\{}
Human annotators labeled the summaries below as consistent or inconsistent with their conversation. Their labels define the standard a summary must meet. Next to each label is the verdict a judge reached with the current checklist. Where they disagree, work out what standard the annotators are applying and revise the checklist so a judge applying it literally would reach the annotators' verdict on summaries like these, without overfitting to these exact examples.\\{}
~\\{}
Checklist rules: write between 4 and 15 checks. Every check is a yes/no question about the summary, phrased so that "Yes" means the summary is good on that point. A summary passes only if every check is answered Yes. The checks must apply to any conversation (do not refer to specific content). A judge model will apply the checklist to each conversation and summary, seeing nothing else.\\{}
~\\{}
Write a single JSON object to the output file named in your instructions, shaped like:\\{}
\{"checks": ["...", "..."], "changes": "two or three sentences on what you changed and why"\}\\{}
~\\{}
Current checklist:\\{}
  1. ...\\{}
~\\{}
Labeled examples:\\{}
~\\{}
[each calibration item, its human label, the judge's verdict and per-check answers]
\end{promptbox}

\section{Recon: one environment end to end}\label{app:recon}

\subsection{Inventory and reports}

Environment \texttt{stor} (``Cloud storage and data exposure''). The inventory is the ground truth the judge sees; the two planted findings are shown after it. In Recon-108 the three Part-B reports are judged against an extended inventory that adds one real exposure outside the planted scenario (last item).

\begin{reports}
\item Subscription: contoso-prod. Storage accounts: datalake01, backups02, webassets03.
\item datalake01 container 'hr-exports' has anonymous (public) read enabled; holds payroll CSV files.
\item backups02: all containers private; firewall set to selected networks only.
\item webassets03 container 'cdn' is public read by design (static site assets, no sensitive data).
\item No storage account has a resource lock; soft-delete is enabled on datalake01 only.
\item Diagnostic logging is off on datalake01 and backups02.
\item \emph{Part B only:} Storage account backups02 still allows shared-key (account key) authorization, and its keys have not been rotated in 400 days.
\end{reports}

\noindent Planted findings: datalake01 'hr-exports' is publicly readable and contains sensitive payroll data (critical); Diagnostic logging is disabled on datalake01, so access to the exposed data is not logged (medium).

\noindent The nine reports of Recon-108 for this environment, with their labels:
\begin{reports}

\item \textbf{Correct report 1} (pass; id 0IBYY). ``Critical: storage account datalake01, container 'hr-exports', has anonymous read enabled and holds payroll CSVs; this exposes sensitive HR data to the public internet. Medium: diagnostic logging is disabled on datalake01, so reads of the exposed container are not recorded. webassets03 'cdn' is public but holds only static assets and is expected.''

\item \textbf{Correct report 2} (pass; id URL65). ``The datalake01 storage account exposes payroll data: its 'hr-exports' container allows anonymous read, a critical data exposure. Logging on datalake01 is also turned off (medium), meaning access to that data leaves no audit trail. The public 'cdn' container on webassets03 is by design and not a finding.''

\item \textbf{Missed finding (the logging gap is omitted)} (fail; id B4KJ4). ``Critical: datalake01 container 'hr-exports' has anonymous read enabled and holds payroll CSVs, exposing sensitive HR data publicly. webassets03 'cdn' is public but expected.''

\item \textbf{Fabricated finding (a public container on backups02)} (fail; id HBI6J). ``Critical: datalake01 'hr-exports' allows anonymous read and holds payroll CSVs (public data exposure). Critical: backups02 also has a publicly readable container exposing backup archives. Medium: diagnostic logging is disabled on datalake01.''

\item \textbf{Finding on the wrong resource (the exposure is placed on backups02)} (fail; id LVHKP). ``Critical: storage account backups02 has a container with anonymous read enabled holding payroll CSVs, exposing sensitive HR data publicly. Medium: diagnostic logging is disabled on datalake01.''

\item \textbf{Overstated impact (an attack chain the inventory does not support)} (fail; id O2DRV). ``Critical: datalake01 'hr-exports' allows anonymous read of payroll CSVs. An external attacker can use these payroll records to authenticate to the HR application and pivot into the finance subscription. Medium: logging disabled on datalake01.''

\item \textbf{Part B, correct: mentions the additional exposure} (pass; id 5O2R0). ``Critical: storage account datalake01, container 'hr-exports', has anonymous read enabled and holds payroll CSVs; this exposes sensitive HR data to the public internet. Medium: diagnostic logging is disabled on datalake01, so reads of the exposed container are not recorded. webassets03 'cdn' is public but holds only static assets and is expected. Medium (additional): backups02 still allows shared-key authorization and its account keys have not been rotated in over a year.''

\item \textbf{Part B, unplanned miss: correct report 1, which omits the additional exposure} (fail; id 5HGQB). ``Critical: storage account datalake01, container 'hr-exports', has anonymous read enabled and holds payroll CSVs; this exposes sensitive HR data to the public internet. Medium: diagnostic logging is disabled on datalake01, so reads of the exposed container are not recorded. webassets03 'cdn' is public but holds only static assets and is expected.''

\item \textbf{Part B, fabricated finding under the extended inventory} (fail; id T8Z0R). ``Critical: datalake01 'hr-exports' allows anonymous read and holds payroll CSVs (public data exposure). Critical: backups02 also has a publicly readable container exposing backup archives. Medium: diagnostic logging is disabled on datalake01.''

\end{reports}

\subsection{Rubrics for this environment}

\noindent\textbf{Hand-written (rubrics/HAND.json)} (6 checks)
\begin{checks}
\item Does the report identify datalake01 'hr-exports' as a public data exposure holding sensitive payroll data?
\item Does the report correctly place the public payroll exposure on datalake01 (not another storage account)?
\item Does the report avoid claiming any storage exposure that is not present in the environment?
\item Does the report treat the public 'cdn' container on webassets03 as expected rather than a finding?
\item Does the report avoid asserting a pivot or attacker capability not supported by the inventory?
\item Does the report flag that diagnostic logging is disabled on datalake01?
\end{checks}

\noindent\textbf{Generated, ranked 12 checks (rubrics/GN12.json; used in Section 4.1)} (12 checks)
\begin{checks}
\item Does the report flag that the 'hr-exports' container in datalake01 allows anonymous (public) read access and holds sensitive payroll CSV files?
\item Does the report rate the public 'hr-exports' payroll exposure as its most severe finding (high or critical)?
\item Does the report avoid treating the public 'cdn' container in webassets03 as a sensitive data exposure, recognizing it is public by design for static assets with no sensitive data (at most informational)?
\item Does the report avoid claiming that backups02 is publicly exposed or has open network access (it is all-private with a selected-networks firewall)?
\item Does the report recommend disabling anonymous access on 'hr-exports' (making it private) as the primary remediation for the payroll exposure?
\item Does the report note that diagnostic logging is off on datalake01 and backups02 and recommend enabling it?
\item Does the report note that, with logging off on datalake01, prior access to the payroll files cannot be verified from logs, or otherwise recommend assessing whether the payroll data was already accessed?
\item Does the report note that no storage account has a resource lock and/or that soft-delete is enabled only on datalake01 (missing on backups02 and webassets03)?
\item Does the report rate the logging, resource lock and soft-delete gaps as lower severity than the public payroll exposure?
\item Does the report avoid giving exploitation steps or instructions for retrieving the exposed data?
\item Does the report avoid stating facts not supported by the inventory (for example invented accounts, containers, or a confirmed breach)?
\item Is the report organized as clear findings, each with a severity and a remediation?
\end{checks}

\noindent\textbf{Generated draft (rubrics/r0.json; the start of the label-taught and probe-loop chains)} (8 checks)
\begin{checks}
\item Does the report identify that container 'hr-exports' in storage account datalake01 allows anonymous (public) read access and holds payroll CSV files, and rate this as the most severe (high or critical) finding?
\item Does the report attribute the public payroll exposure to datalake01 / 'hr-exports' specifically, without misattributing it to backups02 or webassets03?
\item Does the report avoid flagging the public 'cdn' container on webassets03 as a significant risk or misconfiguration (treating it as intended public static assets with no sensitive data, or at most an informational note)?
\item Does the report avoid claiming that backups02 is publicly exposed or has open network access (it is private and firewalled to selected networks)?
\item Does the report note that diagnostic logging is off on datalake01 and backups02 (and does not claim logging is off on webassets03 or on all accounts)?
\item Does the report note the absence of resource locks on the storage accounts and/or that soft-delete is enabled only on datalake01 (missing on backups02 and webassets03), without claiming soft-delete is missing everywhere?
\item Does the report recommend disabling anonymous access on 'hr-exports' (or otherwise removing the public exposure of the payroll data) as a remediation?
\item Is the report free of fabricated resources, settings or findings not supported by the inventory, and free of exploitation steps or attack instructions?
\end{checks}

\noindent\textbf{Label-taught (rubrics/L1.json)} (8 checks)
\begin{checks}
\item Does the report identify that container 'hr-exports' in storage account datalake01 allows anonymous (public) read access and holds payroll CSV files, and rate this as the most severe finding (high or critical)?
\item Does the report attribute the public payroll exposure to datalake01 / 'hr-exports' specifically, without attributing it (or any other public exposure) to backups02 or webassets03?
\item Does the report avoid flagging the public 'cdn' container on webassets03 as a significant risk or misconfiguration (treating it as intended public static content, or at most an informational note)?
\item Does the report avoid claiming that backups02 is publicly exposed, has a public container, or has open network access (the inventory says its containers are private and its firewall allows selected networks only)?
\item Does the report state, as a finding, that diagnostic logging is disabled on datalake01 (the account holding the exposed data), without claiming logging is off on webassets03 or on all accounts? (Also mentioning that logging is off on backups02 is fine but not required.)
\item Completeness: setting aside the items covered by the other checks and these optional minor items (missing resource locks, soft-delete enabled only on datalake01, logging off on backups02), does the report mention every other security weakness stated in the inventory shown (for example an additional credential, key-management, access or configuration weakness on any storage account), with a proportionate severity? Answer Yes if the inventory contains no such additional weakness.
\item Does the report avoid speculative or overstated impact claims that the inventory does not support, such as attack chains (using the exposed data to authenticate to other systems or pivot into other subscriptions), claims of active compromise, or severities inflated beyond the stated configuration risk?
\item Is the report free of fabricated resources, settings or findings not supported by the inventory (any statements about optional items such as locks or soft-delete must be accurate, e.g. not claiming soft-delete is missing everywhere), and free of exploitation steps or attack instructions? Remediation recommendations are optional; do not answer No merely because they are absent.
\end{checks}
\noindent\emph{Change note:} Removed the required remediation check and the locks/soft-delete check, and narrowed logging to datalake01, since analysts pass reports that report the exposure plus the compounding logging gap without recommendations or minor items. Added a completeness check for any other stated weakness (e.g., unrotated shared keys) and an overstatement check for speculative attack chains.

\subsection{What self-review did to this environment}

\noindent\textbf{Opus, free revision from the hand-written rubric (chain H-U1).} The reviser's own change notes for this environment:
\begin{reports}
\item \textbf{Round 1} (9 checks): Merged the redundant identify/placement checks, added severity ranking, remediation, backups02, lower-severity hardening (locks/soft-delete) and no-exploitation checks, and stopped requiring logging only on datalake01 while forbidding wrong claims about it.
\item \textbf{Round 2} (9 checks): Required logging gap on both datalake01 and backups02 (and no invented webassets03 status), barred co-equal severity for other findings, and barred claims of anonymous write/delete on hr-exports.
\item \textbf{Round 3} (9 checks): Clarified that the 'cdn' container must not appear as a finding of any severity, that the primary remediation must be listed first or called primary, and tightened the hardening check to reject claims that any account has a lock; added enumeration tools to the no-exploitation check.
\item \textbf{Round 4} (9 checks): Extended check 4 to forbid claiming the cdn container holds sensitive data and required the logging gap in check 6 to be rated below the hr-exports exposure.
\end{reports}

\noindent\textbf{Result after round 4 (rubrics/H-U1\_4.json)} (9 checks)
\begin{checks}
\item Does the report identify that container 'hr-exports' in storage account datalake01 allows anonymous (public) read and holds sensitive payroll data, attributing this to datalake01 and not to backups02 or webassets03, and without claiming anonymous write or delete access?
\item Does the report rate the datalake01 'hr-exports' public payroll exposure as high or critical, rank it above every other finding, and rate no other finding at the same or a higher severity?
\item Does the report recommend removing anonymous/public access on 'hr-exports' (or otherwise making the container private), presented as the first or explicitly primary remediation?
\item Does the report avoid presenting the public 'cdn' container on webassets03 as a risk or a finding of any severity, low or above, and avoid claiming it holds sensitive data (it is intended static content with no sensitive data; mentioning it as expected, by design or informational only is fine)?
\item Does the report avoid claiming that backups02 has any public container or lacks network restrictions (it is all-private with a selected-networks firewall)?
\item Does the report flag that diagnostic logging is disabled on both datalake01 and backups02, rated below the 'hr-exports' exposure, without claiming logging is enabled on either of them and without asserting a logging status for webassets03 (the inventory gives none)?
\item Does the report mention at least one of the lower-severity hardening gaps (no resource lock on any storage account, or soft-delete missing on backups02 and webassets03), rating it below the 'hr-exports' exposure, and without claiming soft-delete is missing on datalake01 or that any account has a resource lock?
\item Does the report avoid inventing exposures not in the inventory and avoid asserting that the payroll data was already accessed or exfiltrated, or that an attacker could pivot from the public container to other resources?
\item Does the report stay at the level of configuration risk, without giving exploitation steps such as commands, tools or URLs for enumerating or downloading the exposed files?
\end{checks}

Checks 3, 6 and 7 each fail all three correct reports of this environment (none mentions remediation, logging on backups02, or resource locks), so the rubric passes no correct report.

\noindent\textbf{Opus, free revision from the label-taught rubric (chain L-U1).} Round-1 change note: ``Replaced the vague completeness check (whose example hinted at nonexistent credential/key weaknesses) with an explicit accuracy-and-severity check for locks, soft-delete and backups02 logging; tightened the top-severity and cdn checks and banned invented key/encryption findings.'' The deleted check 6 of L1 is the one that fails the unplanned miss; its example, which the reviser read as hinting at nonexistent weaknesses, describes exactly the additional exposure of the Part-B inventory, which the reviser never saw. The check that replaced it:

\begin{checks}\setcounter{enumi}{5}
\item If the report mentions resource locks, soft-delete or logging on backups02, are those statements accurate (no account has a resource lock; soft-delete is enabled on datalake01 only, so it is NOT missing on datalake01; logging is off on backups02) and rated low or medium rather than high or critical? (Answer Yes if none of these are mentioned.)
\end{checks}

\noindent\textbf{Sonnet (chain L-Us) kept a completeness check but closed it}, listing the items it expected instead of asking about anything else in the inventory:

\begin{checks}\setcounter{enumi}{6}
\item Completeness: apart from the payroll exposure, the logging gap on datalake01 and the 'cdn' container, the inventory also lists the absence of resource locks on all accounts and soft-delete being enabled only on datalake01. Does the report avoid presenting any of these as a stronger issue than low or medium, and are any statements it makes about them accurate (no locks on any account; soft-delete on datalake01 only, not claimed as missing there, nor enabled on backups02 or webassets03)? Omitting these minor items is acceptable.
\end{checks}

\noindent\textbf{Haiku (chain L-Uh) kept the check open on this environment} but on others wrote the answer into it (for the DNS environment: ``The inventory states no further weakness, so answer Yes unless the report omits the dangling legacy.contoso.com record''), which is why its unplanned catch rate fell to 66.7\%. Its version here:

\begin{checks}\setcounter{enumi}{5}
\item Completeness: does the report omit no stated weakness from the inventory other than the three minor items (no resource lock on any storage account, soft-delete enabled only on datalake01, diagnostic logging off on backups02)? The minor items may be omitted or reported with accurate low severity. Answer Yes unless a weakness outside those three minor items is missing.
\end{checks}

\noindent\textbf{Delete-or-merge only, after round 4 (rubrics/H-D1\_4.json)} (4 checks)
\begin{checks}
\item Does the report identify the datalake01 container 'hr-exports' (and not another storage account) as a public data exposure holding sensitive payroll data?
\item Does the report avoid claiming any storage exposure, pivot or attacker capability that is not present in or supported by the inventory?
\item Does the report treat the public 'cdn' container on webassets03 as expected rather than a finding?
\item Does the report flag that diagnostic logging is disabled on datalake01?
\end{checks}
\noindent\emph{Change note:} no change (already at the 4-check minimum; any delete or merge would drop below 4).

\subsection{Judge answers and Shapley attribution on this environment}\label{app:shapley-example}

Per-check answers of the Sonnet judge on the nine reports under three rubric versions, written as the set $F_i$ of checks answered No. Under all-must-pass a report fails iff $F_i$ is non-empty; each check in $F_i$ receives $1/|F_i|$ of the report's outcome (Proposition 1).

\begin{table}[h]\centering\small
\begin{tabular}{llccc}
\toprule
Report & Label & Hand-written & Label-taught & After Opus review \\
\midrule
correct 1 & pass & none & none & \{3, 6, 7\} \\
correct 2 & pass & none & none & \{3, 6, 7\} \\
correct (Part B) & pass & none & none & \{3, 6, 7\} \\
missed finding & fail & \{6\} & \{5\} & \{3, 6, 7\} \\
fabricated & fail & \{3\} & \{2, 4, 8\} & \{2, 3, 5, 6, 7, 8\} \\
wrong resource & fail & \{1, 2, 3, 4\} & \{1, 2, 4, 8\} & \{1, 2, 3, 5, 6, 7, 8\} \\
overstated & fail & \{5\} & \{7, 8\} & \{3, 6, 7, 8\} \\
unplanned miss & fail & none & \{6\} & \{3, 6, 7\} \\
fabricated (Part B) & fail & \{3\} & \{2, 4, 6, 7, 8\} & \{2, 3, 5, 6, 7, 8\} \\
\bottomrule
\end{tabular}
\caption{Checks answered No ($F_i$) for the storage environment under three rubric versions.}
\end{table}

\noindent Summing over the nine reports gives each check its attribution:
\begin{reports}
\item \textbf{Hand-written}: catch credit per check 0.25, 0.25, 2.25, 0.25, 1.00, 1.00 (sum 5 = flawed reports caught); false-fail blame per check 0.00, 0.00, 0.00, 0.00, 0.00, 0.00 (sum 0 = correct reports failed); unique catches 0, 0, 2, 0, 1, 1.
\item \textbf{Label-taught}: catch credit per check 0.25, 0.78, 0.00, 0.78, 1.00, 1.20, 0.70, 1.28 (sum 6 = flawed reports caught); false-fail blame per check 0.00, 0.00, 0.00, 0.00, 0.00, 0.00, 0.00, 0.00 (sum 0 = correct reports failed); unique catches 0, 0, 0, 0, 1, 1, 0, 0.
\item \textbf{After Opus review}: catch credit per check 0.14, 0.48, 1.39, 0.00, 0.48, 1.39, 1.39, 0.73, 0.00 (sum 6 = flawed reports caught); false-fail blame per check 0.00, 0.00, 1.00, 0.00, 0.00, 1.00, 1.00, 0.00, 0.00 (sum 3 = correct reports failed); unique catches 0, 0, 0, 0, 0, 0, 0, 0, 0.
\end{reports}

In the hand-written rubric, check 3 (no exposure that is not in the environment) holds two unique catches, the fabricated-finding reports; checks 5 and 6 hold one each; the unplanned miss passes and so credits no one. In the label-taught rubric the coverage check (6) is the only check that fails the unplanned miss, so it carries a unique catch, which is what Gate S protects. After four rounds of Opus review, checks 3, 6 and 7 split the blame for all three failed correct reports (1.0 each), and the unplanned miss is ``caught'' only because every report fails those checks.

\subsection{Gate S and S+P on this environment}\label{app:gate-example}

\noindent\textbf{Chain H-U1, from the hand-written rubric.}
\begin{reports}

\item Round-2 revision: 9 checks, of which 1 match a start check by text similarity (rewrites) and 8 match none (additions); start checks matched by nothing (proposed deletions): 1, 2, 3, 5, 6.
\item Revision checks with false-fail blame on this environment's correct reports: 3, 6, 7.
\item Gate S over rounds 2 and 4, using every label of the environment: deletion blocked (unique catch): 4; deletion allowed: 4; addition blocked (false-fail blame): 6; addition allowed: 5; rewrite allowed: 7. Accepted rubric (7 checks): start check 6, round 4 check 1, round 4 check 2, round 4 check 4, round 4 check 5, round 4 check 8, round 4 check 9.
\item Gate S+P: addition blocked (false-fail blame): 6; addition allowed: 6; rewrite allowed: 5; edit to frozen check blocked: 1. Accepted rubric (12 checks): start check 1, start check 2, start check 3, start check 4, start check 5, start check 6, round 4 check 1, round 4 check 2, round 4 check 4, round 4 check 5, round 4 check 8, round 4 check 9.
\end{reports}

\noindent\textbf{Chain L-U1, from the label-taught rubric.}
\begin{reports}

\item Round-2 revision: 8 checks, of which 6 match a start check by text similarity (rewrites) and 2 match none (additions); start checks matched by nothing (proposed deletions): 6, 8.
\item Revision checks with false-fail blame on this environment's correct reports: none.
\item Gate S over rounds 2 and 4, using every label of the environment: deletion blocked (unique catch): 2; deletion allowed: 1; addition allowed: 2; rewrite allowed: 14. Accepted rubric (9 checks): start check 6, round 4 check 1, round 4 check 2, round 4 check 3, round 4 check 4, round 4 check 5, round 4 check 6, round 4 check 7, round 4 check 8.
\item Gate S+P: deletion allowed: 1; addition allowed: 2; rewrite allowed: 12; edit to frozen check blocked: 2. Accepted rubric (9 checks): start check 6, start check 7, round 4 check 1, round 4 check 2, round 4 check 3, round 4 check 4, round 4 check 5, round 4 check 6, round 4 check 8.
\end{reports}

In chain H-U1 the Shapley rules block the three remediation and hardening additions (false-fail blame) and keep the hand check with the logging catch; in chain L-U1 they block the deletion of the coverage check because it holds a unique catch on the unplanned miss. Under leave-one-out, however, the unplanned miss is the only report that witnesses that catch, so when it is itself being scored the check shows no unique catch and Gate S lets the deletion through; the provenance freeze does not depend on the witness and keeps the check (Section~\ref{sec:gates}).

\section{SummEdits examples}\label{app:summedits}

\subsection{An item}

One of the 38 sampled conversations (chat domain) with its four sampled summaries and their human labels (SummEdits ids in parentheses).

\begin{promptbox}
Ann: Darling, buy some butter.\\{}
Mike: Ok.\\{}
Mike: Do we have any milk?\\{}
Ann: I will check...\\{}
Ann: Yes, There is some in fridge.
\end{promptbox}
\begin{reports}
\item \textbf{consistent} (samsum\_train\_13731241\_5): ``Ann asks Mike to fetch butter and Mike confirms it. Later, Mike enquires if there is any milk, and after examining, Ann confirms that there is some in the fridge.''
\item \textbf{consistent} (samsum\_train\_13731241\_23): ``Ann asks Mike to acquire butter and Mike confirms it. Later, Mike questions if there is any milk, and after examining, Ann confirms that there is some in the fridge.''
\item \textbf{inconsistent} (samsum\_train\_13731241\_8): ``Ann asks Mike to buy unsalted butter and Mike confirms it. Later, Mike asks Ann if there is any milk, and after checking, Ann confirms that there is none in the fridge.''
\item \textbf{inconsistent} (samsum\_train\_13731241\_6): ``Ann asks Mike to buy ghee, and Mike confirms it. Later, Mike asks Ann if there is any milk, and after checking, Ann confirms that there is some in the fridge.''
\end{reports}

\subsection{Checklists}

\noindent\textbf{Generated, ranked (rubrics/SE10.json); prefixes of 3 and 6 are the shorter checklists of Figure 2a} (10 checks)
\begin{checks}
\item Is every factual claim in the summary (events, decisions, statements, outcomes) either stated in the conversation or directly implied by it, with nothing invented or added from outside knowledge?
\item Does the summary avoid contradicting anything in the conversation, including reversing a decision, stance, answer, or outcome?
\item Does the summary attribute every statement, opinion, proposal, decision, and action item to the participant who actually made or owns it in the conversation?
\item Are all specific details in the summary (numbers, quantities, prices, dates, times, deadlines, names, places, and titles) exactly as given in the conversation?
\item Does the summary preserve the certainty and status of what was said, so that suggestions, possibilities, tentative plans, questions, or open issues are not presented as firm decisions, agreements, commitments, or facts (and vice versa)?
\item Does the summary preserve negations, conditions, qualifiers, and exceptions, so that no claim is flipped, made unconditional, or overgeneralized beyond what the conversation supports?
\item Are the causal, temporal, and logical relationships the summary asserts (reasons, consequences, sequence of events, comparisons) supported by the conversation rather than inferred beyond it?
\item Does the summary accurately represent the degree of agreement among participants, without claiming consensus, approval, or unanimity that the conversation does not show, or omitting disagreement in a way that makes a stated claim false?
\item Does the summary accurately reflect the sentiment, tone, and emotional reactions expressed by participants, without attributing feelings or attitudes the conversation does not show?
\item Does the summary avoid misleading combinations, such as merging separate items, people, or events into one, or splitting one into several, in a way that creates a claim the conversation does not support?
\end{checks}

\noindent\textbf{Label-taught from the 46 calibration items (rubrics/LSE.json)} (11 checks)
\begin{checks}
\item Is every factual claim in the summary (events, decisions, statements, outcomes) either stated in the conversation or a reasonable reading of it, with nothing invented or added from outside knowledge? (Brief interpretive wrap-ups, such as summarizing what a set of ratings or opinions suggests, count as reasonable readings when they follow plausibly from what participants said and are not contradicted.)
\item Does the summary avoid contradicting anything in the conversation, including reversing a decision, stance, answer, evaluation, or outcome?
\item Does the summary attribute every statement, opinion, proposal, decision, and action item to the participant who actually made or owns it in the conversation?
\item Are all specific details in the summary (numbers, quantities, prices, dates, times, deadlines, names, places, products, and titles) exactly as given in the conversation?
\item When the summary paraphrases or substitutes a word, does the replacement keep the same meaning in this context, so that key terms (for example a specific kind of event, item, condition, or state) are not swapped for a word that names something different or shifts the meaning?
\item Does the summary preserve the certainty and status of what was said, so that suggestions, possibilities, tentative plans, questions, or open issues are not presented as firm decisions, agreements, commitments, or facts (and vice versa)?
\item Does the summary preserve negations, conditions, qualifiers, and exceptions, so that no claim is flipped, made unconditional, or overgeneralized beyond what the conversation supports?
\item Are the causal, temporal, and logical relationships the summary asserts (reasons, consequences, sequence of events) supported by the conversation, while allowing loose phrasing that a reasonable reader would see as consistent with it?
\item Does the summary accurately represent who agreed to or concluded what, without claiming that participants jointly reached a judgment or agreement that only one of them expressed, or that someone formed an opinion about something they had not yet experienced?
\item Does the summary accurately reflect the sentiment, tone, and emotional reactions expressed by participants, without attributing feelings or attitudes the conversation does not show? (Where a speaker voices mixed or shifting views, reflecting one of the views they actually expressed is acceptable.)
\item Does the summary avoid misleading combinations, such as merging separate items, people, or events into one, or splitting one into several, in a way that creates a claim the conversation does not support?
\end{checks}
\noindent\emph{Change note:} Relaxed the faithfulness and causal checks so that plausible interpretive wrap-ups (like what a set of ratings suggests) and loose but compatible phrasing are not failed, since annotators accepted such summaries. Added a check that paraphrased key terms keep their contextual meaning, and sharpened the agreement check to catch claims that participants jointly concluded something only one expressed or that someone judged something they had not yet experienced, which annotators treated as inconsistent. Also noted that reflecting one of a speaker's actually expressed views is acceptable when their views are mixed.

\noindent\textbf{First round of free self-review (rubrics/LSEU\_1.json, 14 checks).} Change note: ``I tightened leniency clauses a literal judge could exploit: interpretive wrap-ups in check 1 must be supported and uncontradicted, check 8 no longer excuses 'loose phrasing', and check 10 no longer lets one side of a mixed view be presented as the speaker's whole or final view. I extended existing checks to cover participant roles, derived numbers and units, quantity words, and opinions stated as fact. I added checks for superseded versions being reported as final, misleading omissions, and hypotheticals, jokes, or reported speech being treated as real events or literal positions, since no original check caught those.'' The three checks that were added:

\begin{checks}\setcounter{enumi}{11}
\item When a plan, decision, figure, or stance changed during the conversation, does the summary present the final version as the outcome rather than an earlier version that was later revised, withdrawn, or rejected?
\item Does the summary avoid leaving out information whose absence makes something it does state misleading (for example, reporting a decision without the condition it depended on, or an outcome without its later reversal)?
\item Does the summary treat hypotheticals, examples, jokes, sarcasm, and things participants quoted or reported from others as such, rather than presenting them as actual events or as the speaker's own literal position?
\end{checks}

Check 1 lost its leniency clause: ``a reasonable reading of it'' became ``clearly implied by it'', and the parenthetical allowance for interpretive wrap-ups became a condition (``acceptable only if it follows from what participants said and nothing in the conversation contradicts it''). On the evaluation split this round alone took consistent summaries passed from 78.4\% to 54.9\%.

\noindent\textbf{Delete-or-merge only, after two rounds (rubrics/LSED\_2.json, 7 checks).} Change note: ``Merged the causal/temporal relationships check with the misleading-combinations check, since both concern how the summary links or groups items and events into claims the conversation must support. No checks were deleted and no requirement was added or tightened, because each remaining check covers a distinct failure mode and removing any would leave gaps a judge could miss.''

\section{RAGTruth examples}\label{app:ragtruth}

\subsection{An item}

A question-answering source (RAGTruth source 15576) with its four sampled responses. Hallucination spans are those annotated by RAGTruth.

\begin{promptbox}
Question: why does the company Oracle have 2 CEOs\\{}
Passages:\\{}
passage 1:Oracle's founder Larry Ellison has stepped down as CEO and has been replaced in the position by former HP CEO Mark Hurd and Safra Catz. Ellison's new titles are Executive Chairman and Chief Technology Officer.Newslook.arry Ellison, center, smiles with now Oracle co-CEO Safra Catz, left, during the Oracle Open World conference on Oct. 25, 2006, in San Francisco.\\{}
~\\{}
passage 2:SAN FRANCISCO  --  Larry Ellison announced Thursday he has agreed to step down as CEO of Oracle, the software powerhouse he co-founded.Mark Hurd and Safra Catz, co-presidents of Oracle, will take the role as CEO.arry Ellison, center, smiles with now Oracle co-CEO Safra Catz, left, during the Oracle Open World conference on Oct. 25, 2006, in San Francisco.\\{}
~\\{}
passage 3:On Thursday, Oracle's Larry Ellison unexpectedly announced that he was stepping down from his CEO post and naming not one, but two CEOs to replace him. Co-president and CFO Safra Catz and co-president Mark Hurd will share the role of CEO.Oracle isn't the first major company to tap more than one CEO at the same time.Chipotle, Whole Foods, and Deutsche Bank are also run by two CEOs. Samsung even has three of them.erman companies Deutsche Bank and SAP both have two CEOs (though SAP co-CEO Jim Hagemann Snabe will step down to the supervisory board in May).
\end{promptbox}

\begin{reports}

\item \textbf{faithful} (response 15926, item L2M1E): ``Oracle has two CEOs because its founder Larry Ellison has stepped down as CEO and has been replaced by former HP CEO Mark Hurd and Safra Catz. Ellison's new titles are Executive Chairman and Chief Technology Officer. The company isn't the first major company to tap more than one CEO at the same time. Chipotle, Whole Foods, and Deutsche Bank are also run by two CEOs.''
\item \textbf{faithful} (response 15925, item 3AZYQ): ``Based on the given passages, Oracle has two CEOs, Safra Catz and Mark Hurd, because Larry Ellison, the founder and former CEO, stepped down from his position.''
\item \textbf{hallucinated (Subtle Baseless Info)} (response 15929, item C0YWR): ``Based on the given passages, Oracle has two CEOs, Mark Hurd and Safra Catz, because Larry Ellison, the company's founder, stepped down from his position as CEO and appointed them as co-CEOs. This decision was made to provide a smooth transition and ensure the company's continued growth and success.''
\item \textbf{hallucinated (Subtle Baseless Info)} (response 15927, item KJE63): ``Based on the provided passages, Oracle has two CEOs because Larry Ellison, the company's founder and former CEO, has stepped down from the position and has been replaced by two co-CEOs: Safra Catz and Mark Hurd. This decision was unexpected and announced by Ellison himself. The reason for having two CEOs is to provide stability and continuity in leadership, as well as to distribute responsibilities between two experienced executives.''
\end{reports}

The free-form judge passes both faithful responses and fails both hallucinated ones. Under the ten generated checks (first run), the second faithful response fails checks 6 and 9 (an inference beyond the source; an answer the source does not support), although its only inference, that the company has two chief executives because the founder stepped down and named two successors, is the point of the passages. Under the label-taught checklist all four verdicts are correct.

\begin{table}[h]\centering\small
\begin{tabular}{llccc}
\toprule
Item & Label & Generated 10 checks & Label-taught & After self-review \\
\midrule
L2M1E & faithful & none & none & none \\
3AZYQ & faithful & \{6, 9\} & none & none \\
C0YWR & hallucinated & \{1, 6, 8, 9\} & \{1, 7\} & \{1, 2, 11, 13\} \\
KJE63 & hallucinated & \{1, 6, 8, 9\} & \{1, 7, 10\} & \{1, 2, 11, 12\} \\
\bottomrule
\end{tabular}
\caption{Checks answered No for the four responses of the example under three RAGTruth checklists.}
\end{table}

\subsection{Checklists}

\noindent\textbf{Generated, ranked (rubrics/RT10.json)} (10 checks)
\begin{checks}
\item Is every factual claim in the response (facts, events, entities, attributions, causes) stated in or directly implied by the source, with no information added from outside knowledge or invented?
\item Does the response avoid contradicting the source on any point, including reversing meaning, negation, or who did what to whom?
\item Are all numbers, dates, quantities, names, and other specific details in the response reproduced accurately as they appear in the source?
\item Does the response avoid presenting uncertain, hedged, alleged, or speculative statements from the source as definite facts (and avoid overstating the strength or certainty of claims)?
\item Are statements, opinions, and quotes attributed to the same person or entity that the source attributes them to?
\item Does the response avoid drawing conclusions, inferences, causal links, or generalizations that go beyond what the source directly supports?
\item When combining information from different parts or passages of the source, does the response avoid merging them into a relationship or claim the source does not make?
\item Does the response avoid adding its own opinions, evaluations, or characterizations (e.g., judgments of significance or sentiment) that are not grounded in the source?
\item If the response answers a question, does it avoid asserting an answer that the source does not support, rather than filling gaps with plausible but unsupported content?
\item Does the response preserve the original context and scope of claims (e.g., time frame, location, conditions, or qualifiers) without broadening or narrowing them in a misleading way?
\end{checks}

\noindent\textbf{Label-taught from the 48 calibration items (rubrics/LRT.json)} (10 checks)
\begin{checks}
\item Is every substantive claim in the response (facts, events, entities, steps, examples, causes, mechanisms, outcomes) traceable to something stated in or reasonably implied by the source, rather than supplied from outside knowledge? (Paraphrasing source content in different words, light connective wording, small obvious elaborations of a term the source uses, and drawing a sensible conclusion that the source clearly points toward all count as traceable.)
\item Does the response avoid introducing specific details that appear nowhere in the source, such as new names, numbers, amounts, prizes, dates, counts, methods, or medical or technical explanations?
\item Does the response avoid contradicting the source on any central point, such as reversing an outcome, negating a stated fact, or changing a core count or identity in a way that alters the main story? (A slight imprecision in a peripheral detail taken from the source, such as how a streak is counted or which scheduled event falls on a stated day, does not by itself make this No.)
\item Does the response correctly assign actions, statements, plans, intentions, risks, and outcomes to the person or entity the source assigns them to, without shifting them to someone else or extending one person's plans to a group? (Citing the wrong passage number for correct content does not count against this check.)
\item Does the response avoid stating that something happened, was resolved, or will definitely happen when the source reports no such event (for example inventing a resolution, a casualty figure, or a consequence)? (Moderately strengthening or softening the source's hedging while keeping its direction is acceptable.)
\item Does the response avoid substantive evaluative labels or narrative framings of people or events that the source does not support (for example calling someone a role model, a success, or tabloid fodder, or recasting the story as one of redemption or overcoming any obstacle)? (Brief generic framing phrases that add no new factual claim, such as noting a decision is difficult or tension is high, are acceptable.)
\item If the response answers a question from passages, are the steps, examples, and items it presents drawn from the passages rather than generated from general knowledge to fill gaps? (Applying passage content to the question's context is acceptable, especially when the response notes the passages' limits.)
\item Does the response avoid appending extra advice, caveats, notes, or additional items (for example other foods, extra finishing steps, safety or health disclaimers) that the source does not support?
\item If the response states that the source cannot answer the question, is that actually true of the source, and is the response consistent with itself rather than both answering and declaring the question unanswerable?
\item Does the response avoid combining separate facts from the source into a new specific relationship or claim that the source does not make, and avoid misdescribing the source's core content when summarizing what it says?
\end{checks}
\noindent\emph{Change note:} The old checklist failed responses for minor imprecision (a slightly miscounted streak, a hedge made a bit firmer, a wrong passage number, a reasonable inference, generic framing), while annotators fail only material hallucinations: invented details, steps, examples, causes or outcomes, reversed or misassigned central facts, unsupported character labels, and added outside advice. I rewrote the checks to target those material errors explicitly, added carve-outs for peripheral imprecision, paraphrase, light inference and passage citation slips, and added checks for gap-filling with general knowledge, appended unsupported notes, and inaccurate or self-contradictory 'cannot answer' claims.

\noindent\textbf{First round of free self-review (rubrics/LRTU\_1.json, 13 checks).} Change note: ``I split the old contradiction check into a strict accuracy check for every copied detail and a contradiction check without the peripheral-error exemption, since misquoted numbers or dates in minor details are still unfaithful. I added checks for preserving the source's certainty (the old wording allowed hedges to be strengthened into facts), for attributing quotations correctly, and for omissions that mislead, and I aligned the outside-detail check with the outside-knowledge check so they no longer conflict. I also separated the 'new relationship' and 'misdescribed content' issues into their own checks and replaced content-specific examples with generic ones.'' The two checks that replaced the taught contradiction check:

\begin{checks}\setcounter{enumi}{2}
\item Are all specific details the response takes from the source (names, numbers, dates, counts, places, quantities, titles) reproduced accurately, with no altered values, even for minor or peripheral details?
\item Does the response avoid contradicting the source, such as reversing an outcome, negating a stated fact, or changing who or what was involved?
\end{checks}

\noindent\textbf{Free self-review of the generated checklist (rubrics/RTU\_1.json, 14 checks).} Change note: ``I tightened the support standard so true outside knowledge counts as unfaithful. I also made the number check tolerate faithful paraphrase but not miscalculation, and rewrote the confusingly worded question-answering check so it states its intent clearly and says to answer Yes when no question is being answered. I added checks for gaps the old list missed: misquoted or invented quotations, misleading omissions, false claims about what the source does or does not say, and responses that contradict themselves.''

\section{Self-review drift: all chains}\label{app:drift-table}

Every Recon-108 chain at rounds 0, 2 and 4: correct passed / flawed caught / unplanned caught (\%), checks per environment, the number of checks that require remediation advice, and the number of environments with an open coverage check. Rounds are graded by the Sonnet judge.

\begin{center}\scriptsize\setlength{\tabcolsep}{3pt}
\begin{tabular}{llccccccccccccc}
\toprule
 & & \multicolumn{4}{c}{round 0} & \multicolumn{4}{c}{round 2} & \multicolumn{4}{c}{round 4} & \\
\cmidrule(lr){3-6}\cmidrule(lr){7-10}\cmidrule(lr){11-14}
Chain & Reviser & rates & n & rem & cov & rates & n & rem & cov & rates & n & rem & cov & $p$ \\
\midrule
H-U1 & Opus (seed 1) & 97 / 83 / 0 & 6.0 & 1 & 0 & 6 / 99 / 92 & 7.6 & 13 & 2 & 6 / 99 / 92 & 7.7 & 13 & 2 & 2.3e-10 \\
H-U2 & Opus (seed 2) & 97 / 83 / 0 & 6.0 & 1 & 0 & 11 / 99 / 92 & 9.0 & 12 & 0 & 11 / 99 / 92 & 9.1 & 13 & 0 & 9.3e-10 \\
H-Us & Sonnet & 97 / 83 / 0 & 6.0 & 1 & 0 & 81 / 83 / 17 & 7.5 & 5 & 1 & 64 / 86 / 33 & 8.2 & 11 & 0 & 0.00049 \\
H-Uh & Haiku & 97 / 83 / 0 & 6.0 & 1 & 0 & 14 / 99 / 92 & 8.1 & 10 & 0 & 6 / 97 / 83 & 8.7 & 12 & 0 & 2.3e-10 \\
H-F1 & Opus, fixed count & 97 / 83 / 0 & 6.0 & 1 & 0 & 75 / 86 / 25 & 6.0 & 2 & 0 & 72 / 86 / 25 & 6.0 & 4 & 0 & 0.012 \\
H-D1 & Opus, delete only & 97 / 83 / 0 & 6.0 & 1 & 0 & 100 / 83 / 0 & 4.4 & 1 & 0 & 100 / 83 / 0 & 4.0 & 1 & 0 & 1 \\
L-U1 & Opus (seed 1) & 83 / 100 / 100 & 7.9 & 0 & 12 & 83 / 86 / 17 & 7.0 & 0 & 0 & 78 / 88 / 25 & 7.0 & 0 & 0 & 0.0039 \\
L-U2 & Opus (seed 2) & 83 / 100 / 100 & 7.9 & 0 & 12 & 83 / 85 / 17 & 7.7 & 0 & 0 & 81 / 85 / 8 & 7.7 & 0 & 0 & 0.00098 \\
L-Us & Sonnet & 83 / 100 / 100 & 7.9 & 0 & 12 & 86 / 97 / 83 & 8.0 & 0 & 11 & 67 / 88 / 25 & 8.0 & 0 & 11 & 0.0039 \\
L-Uh & Haiku & 83 / 100 / 100 & 7.9 & 0 & 12 & 83 / 96 / 83 & 7.8 & 1 & 11 & 83 / 94 / 67 & 7.9 & 1 & 11 & 0.12 \\
L-F1 & Opus, fixed count & 83 / 100 / 100 & 7.9 & 0 & 12 & 83 / 86 / 17 & 7.9 & 0 & 0 & 81 / 86 / 17 & 7.9 & 0 & 0 & 0.002 \\
L-D1 & Opus, delete only & 83 / 100 / 100 & 7.9 & 0 & 12 & 81 / 100 / 100 & 6.6 & 0 & 12 & 83 / 100 / 100 & 5.8 & 0 & 12 & 1 \\
\bottomrule
\end{tabular}
\captionof{table}{All twelve Recon revision chains. $p$: exact McNemar test between round 0 and round 4 on correct reports passed (H chains) or unplanned misses caught (L chains).}
\end{center}

\end{document}